\documentclass{article}

\usepackage{PRIMEarxiv}

\usepackage[utf8]{inputenc} 
\usepackage[T1]{fontenc}    
\usepackage{hyperref}       
\usepackage{url}            
\usepackage{booktabs}       
\usepackage{amsfonts}       
\usepackage{nicefrac}       
\usepackage{microtype}      
\usepackage{lipsum}
\usepackage{fancyhdr}       
\usepackage{graphicx}       
\graphicspath{{media/}}     
\usepackage{hyperref}
\usepackage{url}
\usepackage{amsthm}
\usepackage{algorithm}
\usepackage{capt-of}
\usepackage{algorithmic}
\usepackage{graphicx}
\usepackage{xcolor}
\newtheorem{proposition}{Proposition}
\newtheorem{theorem}{Theorem}
\newtheorem{lemma}{Lemma}
\newtheorem{corollary}{Corollary}
\theoremstyle{definition}

\usepackage{booktabs}
\newcommand{\ours}{BELT}
\newcommand{\Wdown}{W^{\mathrm{down}}}
\newcommand{\Wup}{W^{\mathrm{up}}}
\newcommand{\yinjun}[1]{\textcolor{brown}{\textbf{[#1]}}}
\newcommand{\eat}[1]{}
\theoremstyle{remark} 

\usepackage{enumitem}
\usepackage{multirow}
\usepackage{multicol}
\usepackage{natbib}
\usepackage{amsmath}

\title{\ours: BEhavior-Aligned Low-Rank KV Cache Compression with Learned Top-k Rank Allocation}

\author{
  Yaoyang Liu \\
  Peking University \\
\And
  Shuwei Shen \\
  JD.com \\
\And
  Jincheng Tang \\
  Peking University \\
\And
  Jinguang Liu \\
  JD.com \\
\AND
  Jinbao Xue \\
  Tencent \\
\And
  Hailong Yang \\
  BeiHang University \\
\And
  Ke Zhang \\
  JD.com \\
\And
  Yinjun Wu \\
  Peking University \\
}

\begin{document}
\maketitle

\begin{abstract}
    Long-context inference makes the key-value (KV) cache a major memory
bottleneck for large language models. We propose \ours, a low-rank
KV cache compression framework. Its key-side objective targets
attention distributions rather than raw keys or pre-softmax logits,
yielding a quadratic surrogate with a closed-form low-rank solution.
A learned global rank allocator uses differentiable soft selection
during training and hard selection at deployment to meet a
user-specified component budget.
Experiments demonstrate improved model quality under aggressive
compression and increased decoding throughput at comparable accuracy.
Evaluated across three LLMs and multiple benchmarks, \ours\ achieves up to
70\% low-rank KV cache compression while outperforming the strongest
baseline by up to 23.3 points on LongBench, and up to $17.8\times$ overall
KV cache reduction when combined with quantization. With Triton-based GPU
kernels and a quantized cache, \ours\ delivers up to $2.7\times$ speedup
for the attention module and $3.3\times$ end-to-end decoding throughput.

\end{abstract}


\section{Introduction}
Large language models (LLMs) have become ubiquitous in modern AI
applications, yet their growing scale---in both model size and context
length---poses significant inference challenges. Modern LLMs support
128K-token contexts or more
\citep{liu2026retrievalattention, stallone2024scaling, xu2025chatqa},
making the {\bf key-value (KV) cache} a critical inference bottleneck.
During autoregressive generation, the cache stores keys and values
for preceding tokens, and its memory footprint grows linearly with
sequence length, potentially exceeding available GPU memory
\citep{jiang2026towards}. Efficient KV cache management is therefore
essential, motivating approaches such as {\bf token eviction}
\citep{zhang2023h2o, liu2023scissorhands}, {\bf quantization}
\citep{hooper2024kvquant, lin2025qserve}, and
{\bf low-rank decomposition}
\citep{chang2024palu, bhatnagar2026starkv}.

Among these, \emph{low-rank decomposition} exploits hidden-dimension
redundancy to reduce KV cache memory. Prior work follows two broad
directions: adapting general-purpose decomposition methods to the
KV cache
\citep{wang2025svd, yuan2023asvd, abbasi2026zero, bai2025erc,
bai2025ressvd, zhou2026s}, and developing KV-specific methods tailored
to attention and caching dynamics
\citep{zhang2024lorc, chang2024palu, lee2025tale, mu2026sals,
ji2026echokv, chang2025xkv}. STAR-KV~\citep{bhatnagar2026starkv}
and A$^3$~\citep{wong2025a3} represent recent state-of-the-art advances
in adaptive rank control and analytical attention approximation,
respectively.

Despite this progress, two challenges still remain.
First, \emph{the compression criterion of the state-of-the-art does not reflect the effect of approximation errors on attention behavior}. Methods based on KV reconstruction do not explicitly account for how key errors interact with queries and the nonlinear softmax map.
As illustrated in Figure~\ref{fig:attention-case}, retaining dominant
singular directions of the KV tensors can distort attention and
change the generated answer.
A$^3$ moves beyond raw key reconstruction by targeting
\emph{pre-softmax attention logits}, which, however, 
is not the same as preserving attention probabilities or the resulting attention output.
{\em This motivates a key-compression criterion that explicitly accounts for post-softmax behavior}.

Second, \emph{allocating ranks across heterogeneous KV projections
under a prescribed global budget remains challenging}.
Uniform allocation ignores differences in compression sensitivity,
while sensitivity-based searches and learned threshold methods offer
different approaches to adaptive allocation. STAR-KV~\citep{bhatnagar2026starkv} learns block-specific thresholds
with a soft compression penalty, without directly enforcing a
prescribed global component budget. Our analysis further identifies these learned thresholds concentrate under certain conditions, potentially
limiting allocation effectiveness.
\begin{figure}[t]
\begin{center}
\includegraphics[width=\textwidth]{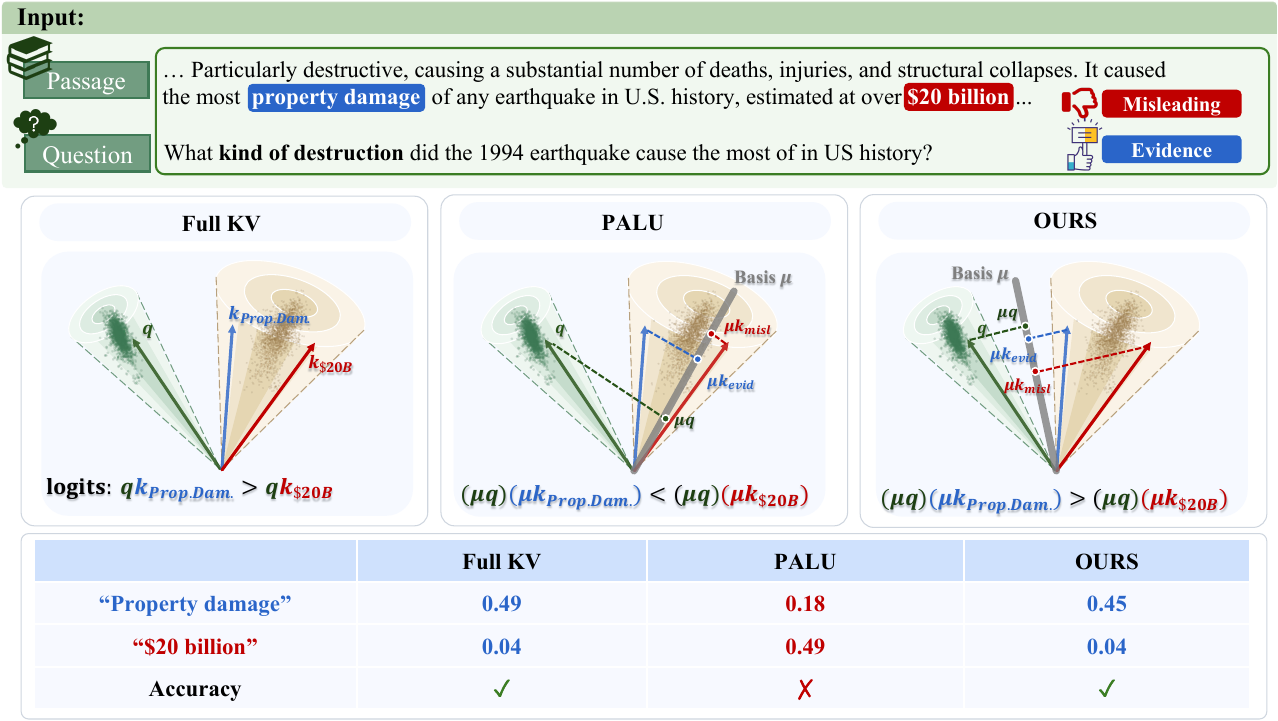}
\end{center}
\caption{\textbf{Left:} Full attention, where each query attends to all context tokens. \textbf{Middle:} State-of-the-art low-rank compression (e.g., vanilla SVD) retains only the dominant KV eigenvector \(\mu\), reducing the query--key product from \(q^\top k\) to \((q^\top \mu)(\mu\top k)\). Since \(\mu\) is derived solely from the input-context tokens and ignores subsequent queries, this distorts the attention distribution and can yield an incorrect answer. \textbf{Right:} The dominant eigen-vector derived by ours aligns the query and key decompositions, preserving the downstream attention output and yielding the correct answer.
}
\label{fig:attention-case}
\end{figure}

We address these challenges from two complementary perspectives.
\emph{First, we design compression objectives to preserve the attention
output of the uncompressed model}.
While our value-side construction is closely related to
A$^3$~\citep{wong2025a3}, our key-side objective targets
attention-distribution preservation rather than pre-softmax logit
reconstruction like A$^3$. By accounting for query-dependent softmax sensitivity
and RoPE interactions, \ours\ better preserves attention behavior
under aggressive compression, as demonstrated empirically.
\emph{Second, we learn ranks across KV blocks under an explicit global
budget}. 
Our allocator learns how to distribute ranks across attention heads
according to their compression sensitivity.
We use a differentiable relaxation of top-$B$ selection for end-to-end
training, followed by hard selection of $B$ components, where $B$ is
the user-specified budget.

Extensive experiments evaluate \ours\ across multiple models,
compression budgets, and tasks.
At 70\% KV cache compression on Qwen3-32B, \ours\ improves the
average score over the strongest evaluated baseline by
23.3 points on LongBench and 13.8 points on RULER.
We additionally evaluate perplexity, zero-shot accuracy, and the
contributions of the compression bases and rank allocation;
combined with 3-bit KV quantization, \ours\ reaches $17.8\times$ KV cache
compression within 2 points of the uncompressed zero-shot accuracy.
System-level measurements show up to $3.26\times$
higher decoding throughput and 87\% lower KV cache memory
than full-cache inference on LLaMA-2-7B with KV quantization.
These results demonstrate the practical benefits of attention-aware
compression and globally coordinated rank allocation.

In summary, our contributions are:
\begin{itemize}[noitemsep, topsep=0pt, left=0pt]
    \item We formulate an attention-output-motivated approach to
    low-rank KV compression, with a key-side metric derived from
    attention-distribution KL rather than pre-softmax logit
    reconstruction.

    \item We derive quadratic compression surrogates with exact
    weighted-SVD solutions, extend the construction to RoPE and
    GQA, and explicitly characterize their approximation errors.

    \item We introduce a learned global rank allocator that
    enforces a relaxed budget during training and an exact
    component-count budget at deployment, while preserving an
    ordered prefix within each KV block.

    \item We empirically evaluate model quality, approximation
    accuracy, rank allocation, and decoding efficiency under
    aggressive KV cache compression.
\end{itemize}

\section{Preliminaries}
\label{sec:prelim}


\subsection{Attention and KV Cache}
\label{sec:prelim:attn}
We consider a single attention head in a Transformer decoder layer~\citep{vaswani2017attention}.
Let \(X \in \mathbb{R}^{n \times d}\) be the stacked representations of the \(n\) preceding tokens stored in the KV cache. For any input token \(x \in \mathbb{R}^{1 \times d}\), each head computes its query, key, and value projections as
\begin{small}
    \begin{align*}
    q_x = xW_Q,\quad k_x = xW_K,\quad v_x = xW_V,        
    \end{align*}
\end{small}
where \(W_Q, W_K, W_V \in \mathbb{R}^{d \times d_h}\) and \(d_h\) is the per-head dimension.
During autoregressive generation, the queries, keys, and values of the \(n\) tokens are obtained from \(X\) as
\begin{small}
    \begin{align*}
    Q = X W_Q, \quad K = X W_K,\quad V = X W_V        
    \end{align*}
\end{small}
so the KV cache stores \(K, V \in \mathbb{R}^{n \times d_h}\), which is interacted with $q$ for attention calculation as follows:
\begin{small}
\begin{align}\label{eq: attention_out}
\begin{split}
&o_{x}(W_K,W_V) = a_x(W_K)\,V\,W_O \in \mathbb{R}^{1 \times d} = a_x(W_K)\,X\,W_V\,W_O \in \mathbb{R}^{1 \times d},\\
&\text{where }
a_x(W_K) = \operatorname{Softmax}\!\left(\frac{q_x\,K}{\sqrt{d_h}}\right) = \operatorname{Softmax}\!\left(\frac{q_x\,X W_K}{\sqrt{d_h}}\right) \in \mathbb{R}^{1 \times n}, \text{ and } W_O \in \mathbb{R}^{d_h \times d}
\end{split}
\end{align}    
\end{small}

In this formula, $a_x(W_K)$ is the attention probability vector, where we explicitly highlight its reliance on $W_K$. Similarly, we use $o_{x}(W_K,W_V)$ to denote the final attention output which relies on both $W_k$ and $W_V$.  In practice, Rotary Position Embedding (RoPE)~\citep{su2024roformer} is applied to queries and keys as a non-linear, position-dependent embedding before attention. More detailed description for attention mechanism with RoPE and how \ours\ is generalized to RoPE is Appendix \ref{app: rope_extension}. 

State-of-the-art large language models typically employ multiple attention heads per layer with two widely adopted strategies: Multi-head attention (MHA) runs \(h\) such heads in parallel, while grouped-query attention (GQA) shares each key--value pair among a group of query heads. Our framework aims to compress KV cache at individual heads, which thus applies to both. We start by assuming the use of MHA by default and discuss the extension of \ours\ to GQA in Appendix \ref{app:gqa}.

\subsection{Low-Rank KV Cache Compression}
\label{sec:prelim:lowrank}

Low-rank compression reduces the KV cache size by approximating each key/value projection matrix \( W_K \in \mathbb{R}^{d \times d_h}, W_V \in \mathbb{R}^{d \times d_h} \) with a rank-\( r \) factorization (\( r < d_h \)):
\begin{small}
    \begin{align*}
    W_K \approx \hat{W}_K = \Wdown_K \Wup_K, W_V \approx \hat{W}_V = \Wdown_V \Wup_V        
    \end{align*}
\end{small}
where \( \Wdown_K, \Wdown_V \in \mathbb{R}^{d \times r} \), \( \Wup_K, \Wup_V \in \mathbb{R}^{r \times d_h} \) and $\hat{W}_K, \hat{W}_V \in \mathbb{R}^{d \times d_h}$, which all have rank at most $r$. For an input token sequence \( X\), we project it using $\hat{W}_K$ and $\hat{W}_V$ as follows:
\begin{small}
    \begin{align*}
        X W_K \approx X \Wdown_K \Wup_K, X W_V \approx X \Wdown_V \Wup_V
    \end{align*}
\end{small}
where \( X \Wdown_K , X \Wdown_V \in \mathbb{R}^{n \times r} \) are the compressed KV cache. Instead of storing the full \( K, V \in \mathbb{R}^{n \times d_h} \), we cache only the compressed representations \( K' = X \Wdown_K \) and \( V' = X \Wdown_V \).
This reduces the per-token memory footprint from \( d_h \) to \( r \) during inference. Following prior studies, e.g., \cite{bhatnagar2026starkv, chang2024palu}, $\Wdown_K$, $\Wdown_V$, $\Wup_K$ and $\Wup_V$ are learned with a small calibration dataset, $\mathcal{X}$. 


\subsection{Inference with Compressed Cache}
\label{sec:prelim:absorption}

The overall inference process with low-rank decomposition remains largely the same as that with no compression. The only difference is at the attention computation phase, where low-rank decomposition requires that the input sequence $X$ is first projected down using $\Wdown_K$ and $\Wdown_V$, which is then cached throughout the inference phase, and projected back to $d_h$ dimensions using $\Wup_K$ and $Wup_V$ respectively.
Weight absorption~\citep{liu2024deepseekv2,chang2024palu} further accelerates this process by merging $\Wup_V$ into the subsequent projection for value tensors. Specifically, the attention output defined in \eqref{eq: attention_out} changes from \(a_x(W_K)\,V\,W_O\) to \(a_x(W_K)\,V'\,\Wup_V\,W_O\). Since both \(\Wup_V\) and \(W_O\) are fixed, precomputing \(W_O' = \Wup_V W_O \in \mathbb{R}^{r \times d}\) yields \(a_x(W_K)\,V'\,W_O'\), eliminating the need of multiplying projecting $V'$ back to $d_h$ dimensions at each step. In contrast, a similar absorption does not apply to the key tensor since \(\Wup_K\) cannot be absorbed with the RoPE computation due to its non-linearity nor commute it to merge it with other matrices like the query projection \(W_Q\). 
\section{The~\ours~Framework}
\label{sec:method}

In this section, we first present low-rank decomposition within one head under a given rank budget in Section~\ref{sec:single_head_decompose}. We then generalize the method to multiple heads under a global budget in Section~\ref{sec:method:alloc}, and finally combine rank allocation with quantization in Section~\ref{sec:method:quant}.

\subsection{Low-rank decomposition within one head}\label{sec:single_head_decompose}

\subsubsection{Decomposition objective}
\label{sec:decomposition_objective}

Unlike methods that directly minimize key/value reconstruction error,
\ours\ seeks low-rank projections whose induced attention output remains close to that of the original model, which is formulated as the following loss function defined over every input $x$ from the calibration data $\mathcal{X}$: 
\begin{small}
\begin{equation*}
\begin{aligned}
J(\widehat W_K,\widehat W_V) = {\sum_x
\left\|
  o_{x}(W_K,W_V)
  -o_{x}(\widehat W_K,\widehat W_V)
\right\|_2^2}, \widehat W_K^*,\widehat W_V^*
=
\operatorname*{arg\,min}_{\substack{
  \operatorname{rank}(\widehat W_K)\le r_K,
  \operatorname{rank}(\widehat W_V)\le r_V}}
{\mathcal J(\widehat W_K,\widehat W_V)} .
\end{aligned}
\label{eq:joint_kv_objective}
\end{equation*}    
\end{small}



However, jointly solving $\widehat W_K^*,\widehat W_V^*$ in $J(\widehat W_K,\widehat W_V)$ is difficult due to their inter-dependency. Indeed, the above objective, $\mathcal J(\widehat W_K,\widehat W_V)$, can be relaxed to the summation of the following two objectives, which decouples the optimization of $\widehat W_K^*$,$\widehat W_V^*$:
\begin{small}
    \begin{align}
    \begin{split}
& \mathcal J(\widehat W_K,\widehat W_V)
\le
2\mathcal L_V(\widehat W_V)
+
4M^2 \cdot 
\mathcal L_{K}(\widehat W_K),\\
& \text{where } \mathcal L_V(\widehat W_V)
:= \sum\nolimits_{x}
\left\|
 a_{x}X(W_V-\widehat W_V)W_O
\right\|_2^2,\\
& \mathcal L_K(\widehat W_K)
:= \sum\nolimits_{x}
D_{\mathrm{KL}}\!\left(
 a_{x}(W_K)\,\middle\|\,a_{x}(\widehat{W}_K)
\right), M_{X,t}:=\max\nolimits_i\|(\widehat Z_{X,t})_{i,:}\|_2\le M.
\end{split}
    \end{align}
\end{small}

The detailed derivation of the above formula is delayed to Appendix \ref{objective_derivation}.

\subsubsection{Solving $\mathcal{L}_V(\widehat W_V)$}
\label{sec:value_bsvd}

Indeed, we can derive a closed-form $\widehat W_V$ by solving $\mathcal{L}_V(\widehat W_V)$. To achieve this, we first transform $\mathcal{L}_V(\widehat W_V)$ into the following surrogate objective.

\begin{lemma}
\label{lemma:lv_trans_eps}
Let
$
G_V^0
=
\sum_x X_x^\top a_x^\top a_xX_x, 
H_V^0=W_OW_O^\top.
$
For $\epsilon>0$, let $
G_V=G_V^0+\epsilon I_d=L_V^\top L_V, 
H_V=H_V^0+\epsilon I_{d_h}=C_VC_V^\top
$
be their positive-definite regularizations and Cholesky
factorizations.
Define
\begin{small}
    \begin{align*}
    \mathcal L_V^\approx(\widehat W_V)
=
\left\|L_V(W_V-\widehat W_V)C_V\right\|_F^2.        
    \end{align*}
\end{small}
Then, with $E=W_V-\widehat W_V$, $
\mathcal L_V^\approx(\widehat W_V)
-
\mathcal L_V(\widehat W_V)
=
\Delta_{\epsilon,V}(E)\ge0,
$

where $\Delta_{\epsilon,V}(E)
=
\epsilon\sum_x\|a_xX_xE\|_2^2
+
\epsilon\|EW_O\|_F^2
+
\epsilon^2\|E\|_F^2.
$
In particular, for fixed $E$ and calibration statistics,
$\Delta_{\epsilon,V}(E)=O(\epsilon)$ as $\epsilon\to0$.
\end{lemma}

The proof of Lemma \ref{lemma:lv_trans_eps} is given in Appendix \ref{app: value_lemma}. Subsequently, by performing some essential linear algebric transformations and applying Eckart--Young--Mirsky theorem \cite{eckart1936approximation, mirsky1960symmetric}, 
The following theorem gives a closed-form global minimizer of the regularized value surrogate $\mathcal L_V^\approx$.

\begin{theorem}
\label{thm:value_bsvd}
let $L_VW_VC_V = \bar U_V\bar\Sigma_V\bar R_V^\top$ be a thin singular value decomposition, where $\bar\Sigma_V =
\operatorname{diag}
(\bar\sigma_{V,1},\ldots,\bar\sigma_{V,m}),
m=\min(d,d_h),
$
and
$\bar\sigma_{V,1}\ge\cdots\ge\bar\sigma_{V,m}\ge0$.
Let $\bar U_{V,1:r}$, $\bar\Sigma_{V,1:r}$, and $\bar R_{V,1:r}$ denote the
components associated with the largest $r$ singular values. Then one
global minimizer of $\min\nolimits_{\operatorname{rank}(\widehat W_V)\le r}
\mathcal L_V^{\approx}(\widehat W_V)$
is
\begin{small}
\begin{equation}
\widehat W_{V,r}^*
=
L_V^{-1}
\bar U_{V,r}\bar\Sigma_{V,r}\bar R_{V,r}^\top
C_V^{-1}.
\label{eq:value_bsvd_solution}
\end{equation}    
\end{small}

The corresponding projection matrix, 
$\Wdown_{V,1:r}
=
L_V^{-1}\bar U_{V,1:r}\bar\Sigma_{V,1:r}^{1/2}$, and $\Wup_{V,1:r}
=
\bar\Sigma_{V,1:r}^{1/2}\bar R_{V,1:r}^\top C_V^{-1}$.
\end{theorem}

The proof is provided in
Appendix~\ref{app: value_theorem}. The above decomposition via solving $\mathcal{L}_V(\widehat W_V)$ is closely related to A$^3$~\citep{wong2025a3}; we do not claim the weighted-SVD solution itself as a distinct contribution. Our principal distinction lies in the key-side decomposition described next.
\subsubsection{Solving $\mathcal{L}_K(\widehat W_K)$}
\label{sec:key_bsvd}

Unlike the value objective, $\mathcal L_K(\widehat W_K)$ contains the nonlinear softmax map. Hence, directly solving $\mathcal L_K(\widehat W_K)$ cannot yield a closed-form solution. To mitigate this issue, we first approximate $\mathcal{L}_K(\widehat{W}_K)$ by a quadratic objective using Taylor expansion \cite{rudin1976principles} and Kronecker decomposition \cite{wilkinson1979kronecker}, which only incurs a small approximation error. Before formally presenting how this approximation is performed in Theorem \ref{thm:lk_kron_rank1}, we first compute the following intermediate quantities, taking as input each query token \(x \in X\) and its corresponding attention probability distribution $a_x$ calculated using \eqref{eq: attention_out}:
\begin{small}
    \begin{align}\label{eq: lk_intermediate_res}
    J_x=\operatorname{diag}(a_x)-a_x^\top a_x,, A_x:=J_x\operatorname{diag}(1/a_x)J_x,
G_{K,x}^{\epsilon}:=\frac{1}{2d_h}X^\top A_xX+\epsilon I_d,
    \end{align}
\end{small}

\begin{theorem}\label{thm:lk_kron_rank1}
Let $M_K:=\sum_x q_xq_x^\top\otimes G_{K,x}^{\epsilon}
\in\mathbb R^{(d_hd)\times(d_hd)}$,
and approximate it using its Kronecker rank-1 approximation
$M_K\approx A_K\otimes B_K,A_K\in\mathbb R^{d_h\times d_h},\quad A_K\succeq0, B_K\in\mathbb R^{d\times d},\quad B_K\succeq0.$
For \(\epsilon>0\), let $B_K+\epsilon I_d=L_K^\top L_K,
A_K+\epsilon I_{d_h}=C_KC_K^\top$ be Cholesky decompositions. Then \(L_K\in\mathbb R^{d\times d}\) and \(C_K\in\mathbb R^{d_h\times d_h}\) are invertible. 
Define the surrogate objective
\begin{small}
\begin{align*}
\mathcal L_K^{\approx}(\widehat W_K)
:=
\left\|
L_K(W_K-\widehat W_K)C_K
\right\|_F^2.
\end{align*}
\end{small}
This is a quadratic function of \(E:=W_K-\widehat W_K\), which incurs an approximation error
\begin{small}
\begin{align*}
|\mathcal L_K(\widehat W_K)
-
\mathcal L_K^{\approx}(\widehat W_K)|
\leq
\underbrace{
|\operatorname{vec}(E)^\top
\left(
M_K-A_K\otimes B_K
\right)
\operatorname{vec}(E)|
}_{\text{Kronecker truncation error}}
+
\underbrace{
\mathcal E_{\mathrm{Taylor}}(E)
}_{\text{Taylor approximation error}}
+
O(\epsilon),
\end{align*}
\end{small}
where \(\mathcal E_{\mathrm{Taylor}}(E)\) represents the overall Taylor approximation error.
\end{theorem}

The proof and explicit approximation error terms are provided in
Appendix~\ref{app: lk_kron_rank1_proof}, and
Algorithm~\ref{alg:lk_kron_rank1} in
Appendix~\ref{app: rank_1_approx_algo} details the construction of
$\mathcal L_K^{\approx}$.
Our empirical analysis in Section~\ref{sec: exp} verifies that both Taylor and Kronecker errors are consistently
small at the low-rank solutions returned by our method.
These results support the accuracy of
$\mathcal L_K^{\approx}$ as an approximation to $\mathcal L_K$
at those solutions under the evaluated compression settings.
They do not establish a uniform approximation guarantee over all
rank-constrained candidates.

Theorem \ref{thm:lk_kron_rank1} yields a weighted quadratic surrogate with the same
algebraic form as the one for $\mathcal{L}_V(\widehat{W}_V)$.
Applying the argument of Theorem \ref{thm:value_bsvd} therefore gives a closed-form minimizer of $\mathcal L_K^{\approx}$ (see Corollary \ref{thm:key_bsvd} in Appendix \ref{app: corollory_key}).

\subsection{Learned Rank Allocation across heads}
\label{sec:method:alloc}

An LLM contains key and value projections across many heads and layers, and their sensitivity to low-rank compression varies substantially. We therefore allocate a global singular component budget across these heads. In this section, we concatenate all $W_K$'s and $W_V$'s across all heads across all layers and index them using $i\in\{1,\ldots,N\}$. Plus, we use $m_i$ to denote the total numbers of singular components for each $i$, and $B$ is the total number of retained components across all $i$'s. We require $N\le B\le\sum_i m_i$ and retain at least one component for each $i$.

To address this problem, we propose to rank all singular components globally throughout training, as summarized in Algorithm~\ref{alg:alloc}. 


\paragraph{Budget-constrained relaxation.}
Each block $i$ has a learnable score $s_i$, and its $j$-th
singular component receives score $t_{i,j}=s_i-j$.
We reserve the first component of every block in both training
and deployment.
For a total budget $B$, the remaining budget is $b=B-N$.
During training, we use
\begin{small}
    \begin{align*}
    g_{i,1}=1,
\qquad
g_{i,j}
=
\sigma\left((t_{i,j}-\alpha)/{\tau}\right),
\quad j\ge2,        
    \end{align*}
\end{small}
where the shared cutoff $\alpha$ is solved using the Newton's method \cite{kelley2003solving} so that
$\sum\nolimits_i\sum\nolimits_{j=2}^{m_i}g_{i,j}=b$. 

We differentiate through this constraint using implicit
differentiation.
Thus training preserves an exact relaxed gate-mass budget,
rather than an exact matrix-rank budget. More training details are delayed to Appendix \ref{app:budget-gates}. 


\paragraph{Hard allocation.}
After training, we retain the first component of every block
and the $b$ remaining components with the largest scores
$t_{i,j}$, using deterministic tie-breaking.
The resulting ranks satisfy
\begin{small}
    \begin{align*}
    1\le r_i\le m_i,
\qquad
\sum\nolimits_i r_i=B.        
    \end{align*}
\end{small}



\begin{algorithm}[!h]
\small
\caption{Learned global rank allocation across KV blocks}
\label{alg:alloc}
\begin{algorithmic}[1]
\STATE {\bfseries Input:} budget $N\le B\le\sum_i m_i$,
calibration set $\mathcal X$, temperature $\tau>0$.
\STATE {\bfseries Output:} ranks $\{r_i\}_{i=1}^{N}$
with $\sum_i r_i=B$.
\STATE Compute and freeze all transformed SVD components
using Theorem~\ref{thm:value_bsvd} and
Corollary~\ref{thm:key_bsvd}.
\STATE Set $b=B-N$; initialize one score $s_i$ per block.
\label{line:alloc_init}
\WHILE{not converged}
    \STATE Set $t_{i,j}=s_i-j$ and reserve $g_{i,1}=1$.
    \STATE Solve the shared cutoff $\alpha$ so that
    $\sum_i\sum_{j=2}^{m_i}
    \sigma\!\left((t_{i,j}-\alpha)/\tau\right)=b.$ using the Newton's method.
    \label{line:solve_alpha}
    \STATE Set $g_{i,j}=\sigma((t_{i,j}-\alpha)/\tau)$
    for $j\ge2$ and gate the fixed transformed singular values.
    \STATE Compute $\mathcal L_{\mathrm{KD}}$ on $\mathcal X$;
    update $\{s_i\}$ by implicitly differentiating through $\alpha$.
\ENDWHILE
\STATE Recompute the scores; retain each block's first component
and the top-$b$ remaining components, breaking ties deterministically.
\label{line:alloc_hard}
\STATE Set $r_i$ to the number retained in block $i$.
\RETURN $\{r_i\}_{i=1}^{N}$.
\end{algorithmic}
\end{algorithm}

\subsection{Integration with Quantization}
\label{sec:method:quant}
To further compress the KV cache, \ours{} can be combined with existing
low-rank KV quantizers without modification. We integrate Palu~\citep{chang2024palu}
and STAR-KV~\citep{bhatnagar2026starkv}. Since the leading coordinates of the
compressed representation dominate the quantization range, both quantizers
apply an orthogonal transform---a norm-preserving change of basis \(R\) with
\(R^\top R=I\) that mixes coordinates to spread outliers---before
quantization. We fuse these transforms into the factors offline, adding no
inference cost. Arbitrary learned ranks are handled with orthogonal matrices,
and no retraining is required. Full details are in Appendix~\ref{app:quant}.

\section{Experiments}\label{sec: exp}

\subsection{Experiment Settings}
\paragraph{Models and Tasks.} 
We evaluate on four LLMs: LLaMA-2-7B~\cite{llama2}, LLaMA-3.1-8B-Instruct~\cite{llama3}, Qwen3-32B~\cite{yang2025qwen3}, and MPT-7B. They span both MHA (LLaMA-2-7B, MPT-7B) and GQA (LLaMA-3.1-8B-Instruct, Qwen3-32B), as well as RoPE (first three) and NoPE (MPT-7B). For accuracy, we measure perplexity on WikiText-2~\cite{wikitext-2} and C4~\cite{c4}, and zero-shot accuracy on six tasks via LM-Evaluation-Harness~\cite{eval-harness}. For long-context evaluation, we additionally test on LongBench~\cite{longbench-bench} and RULER~\cite{ruler-bench}, covering diverse document-level and structured tasks.

\paragraph{Baseline. } We compare \ours\ against two types of baselines for low-rank KV cache compression: (1) general low-rank decomposition methods for LLM compression, including SVD-LLM~\cite{wang2025svd}, ASVD~\cite{yuan2023asvd}, Zero Sum SVD~\cite{abbasi2026zero}, and ERC-SVD~\cite{bai2025erc}; and (2) KV-specific low-rank methods, including PALU~\cite{chang2024palu} and the state-of-the-art STAR-KV~\cite{bhatnagar2026starkv} and A3 \cite{wong2025a3}. Other baselines such as S2VD~\cite{zhou2026s} are excluded due to code unavailability.


\subsection{Accuracy Analysis}

\begin{figure}[h]
\centering
\includegraphics[width=\textwidth]{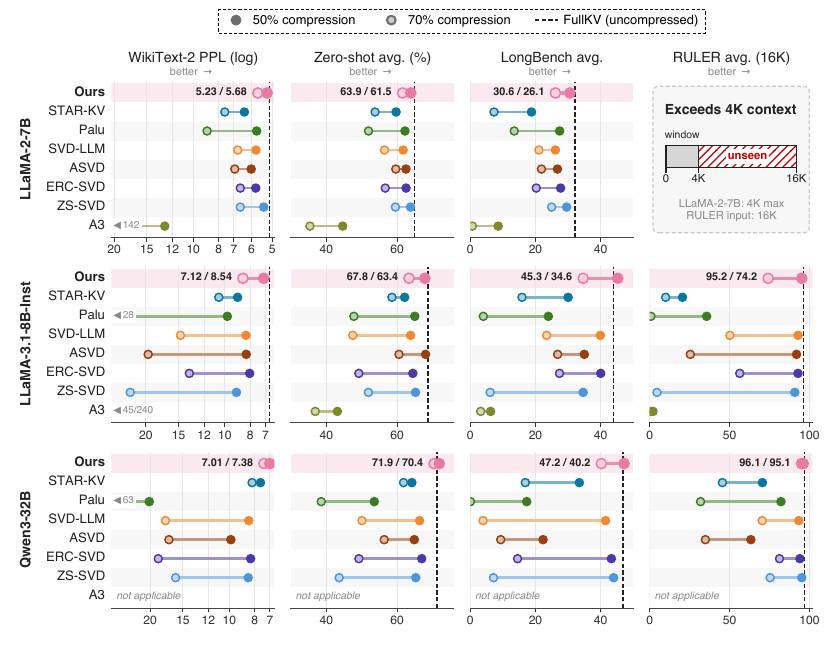}
\caption{
Main results at 50\% (dark) and 70\% (light) KV cache compression. Each row connects one method's two rates; the dashed line marks the uncompressed model. All axes point right toward better results; perplexity is reversed and logarithmic, with out-of-range values printed at the left end. Numbers beside \ours{} denote 50\%/70\% results. RULER uses 16K tokens, exceeding LLaMA-2-7B's 4K context. A3 is not applicable to Qwen3-32B, whose query/key normalization lies between the projections and RoPE. Per-task results are in Appendix~\ref{per-task-main}.
}
\label{fig:main_results}
\end{figure}

\paragraph{Perplexity and Zero-Shot Evaluation.} 
Figure~\ref{fig:main_results} shows that \ours{} maintains model quality under aggressive KV cache compression. On LLaMA-2-7B and LLaMA-3.1-8B-Instruct, \ours{} consistently achieves the best perplexity and zero-shot accuracy at both 50\% and 70\% compression, with only marginal degradation from the uncompressed model. \textbf{In contrast, baselines degrade sharply at 70\% compression, with several collapsing entirely.} \textbf{These results demonstrate that \ours{} preserves generation quality even at high compression rates.}

\paragraph{Long Context Evaluation.}
Figure~\ref{fig:main_results} shows that \ours{} achieves the best average score at both rates, surpassing even the uncompressed LLaMA-3.1-8B at 50\% compression. \textbf{Baselines competitive on LM-Eval-Harness collapse here, confirming that short-context metrics underestimate long-context damage.} \textbf{Retrieval-intensive tasks (Qasper, TriviaQA) are the most compression-sensitive and where \ours{} gains the most}, while summarization remains insensitive---consistent with our motivation that query--key matching dominates retrieval.

\paragraph{Comparison with Fine-tuning.}
\ours{} constructs compressed KV projections without fine-tuning
model weights, whereas the methods in Table~\ref{tab:finetune}
additionally update weights on LLaMA-3.1-8B-Instruct.
{\bf Fine-tuning constantly improves the results of all methods, but baselines still fall behind
\ours{}.}
Direct FT trains low-rank key and value projection matrix from a
random orthogonal initialization.
Despite using the same objective and training data as
\ours{}+FT, it achieves substantially lower accuracy, which can potentially be attributed to the local minimum issue.


\begin{table*}[h]
\centering
\caption{
Fine-tuning the low-rank projection matrix on LLaMA-3.1-8B-Instruct. 
Bold marks the best result at each compression rate.
}
\label{tab:finetune}
\small
\resizebox{\textwidth}{!}{
\begin{tabular}{lcccccccccc}
\toprule
\multirow{2}{*}{\textbf{Method}} & \multirow{2}{*}{\textbf{Comp (\%)}} & \multicolumn{2}{c}{\textbf{Perplexity} $\downarrow$} & \multicolumn{7}{c}{\textbf{Zero-Shot Accuracy (\%)} $\uparrow$} \\
\cmidrule(lr){3-4} \cmidrule(lr){5-11}
 & & Wiki2 & C4 & OBQA & PIQA & ARC-e & ARC-c & Hella & Wino & Avg. \\
\midrule
\multirow{2}{*}{ZS-SVD} & 50 & 7.72 & 14.70 & 43.20 & 77.53 & 77.90 & 49.66 & 74.66 & \textbf{72.61} & 65.93 \\
 & 70 & 9.34 & 22.51 & 35.60 & 71.98 & 65.36 & 38.99 & 64.90 & 65.51 & 57.06 \\
\cmidrule(lr){1-11}
\multirow{2}{*}{STAR-KV} & 50 & 8.31 & 14.82 & 40.20 & 76.66 & 76.01 & 47.61 & 71.65 & 66.06 & 63.03 \\
 & 70 & 9.73 & 18.64 & 37.20 & 74.16 & 70.24 & 41.89 & 65.71 & 61.64 & 58.48 \\
\cmidrule(lr){1-11}
\multirow{2}{*}{Direct FT} & 50 & 8.66 & 17.56 & 37.20 & 75.95 & 72.43 & 43.34 & 68.56 & 59.67 & 59.53 \\
 & 70 & 10.65 & 27.99 & 33.60 & 71.11 & 59.30 & 31.14 & 58.94 & 53.99 & 51.35 \\
\cmidrule(lr){1-11}
\multirow{2}{*}{Ours+FT} & 50 & \textbf{6.96} & \textbf{12.63} & \textbf{44.20} & \textbf{79.87} & \textbf{80.01} & \textbf{54.95} & \textbf{78.24} & 72.45 & \textbf{68.29} \\
 & 70 & \textbf{7.75} & \textbf{15.11} & \textbf{42.80} & \textbf{77.31} & \textbf{75.34} & \textbf{49.40} & \textbf{75.00} & \textbf{69.69} & \textbf{64.92} \\
\bottomrule
\end{tabular}
}
\end{table*}

\subsection{System-level Performance Analysis}
We measure decoding speed on LLaMA-2-7B in FP16 on a single NVIDIA RTX PRO 6000 (96\,GB). The baseline is PyTorch SDPA over the full cache, which dispatches to FlashAttention-2~\citep{dao2023flashattention2}. Because \ours{} at 70\% compression matches or exceeds the quality of STAR-KV, Palu, and ERC-SVD at 50\% (Table~\ref{tab:main_results}), we compare \ours{} at 70\% against these baselines at 50\%. More details about these configurations are given in Appendix \ref{app:system_configure}.

\paragraph{Attention Latency.} Figure~\ref{fig:attn_speedup} reports speedups at roughly matched accuracy, where \ours{}'s better accuracy--speed trade-off yields faster performance at the same recall. \ours{} consistently outperforms SDPA and all baselines, with its advantage growing at longer contexts, while STAR-Q further boosts the speedup. STAR-KV and Palu remain nearly flat below 1.5$\times$, and \ours{} at 70\% compression already surpasses all baselines at 50\%, whereas ERC-SVD is even slower than the uncompressed baseline. {\bf Further runtime breakdown in Appendix~\ref{app:system} shows that attention computation dominates the overhead and that \ours{} substantially accelerates it.}



\paragraph{Generation Throughput.} Table~\ref{tab:throughput} reports full-model decoding throughput without prefill, using the largest power-of-two batch that fits in memory. By storing the \ours{} cache in a small fraction of the FP16 size, STAR-Q enables substantially larger batch sizes than the baselines, yielding a large throughput gain over SDPA, whereas STAR-KV improves throughput only modestly. Even without STAR-Q, \ours{} still achieves considerable speedups.


\begin{figure}[t]
\centering
\begin{minipage}[t]{0.505\textwidth}
\vspace{0pt}
\centering
\includegraphics[width=\linewidth]{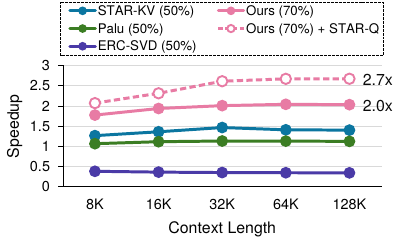}
\caption{Speedup of the attention module over PyTorch SDPA for one decoding step of LLaMA-2-7B at batch size 16, measured over all 32 layers; layers left uncompressed run SDPA. Baselines use 50\% compression and Ours uses 70\%; the dashed line adds STAR-Q.}
\label{fig:attn_speedup}
\end{minipage}\hfill
\begin{minipage}[t]{0.475\textwidth}
\vspace{0pt}
\centering
\captionof{table}{Quantization results on Qwen3-32B: average KV bit-width, KV cache compression ratio, and average zero-shot accuracy over the six commonsense tasks.}
\label{tab:quantization}
\small
\setlength{\tabcolsep}{4pt}
\begin{tabular}{l|c|c|c}
\toprule
\textbf{Method} & \textbf{Bit} & \textbf{Comp.} & \textbf{Avg. (\%)} \\
\midrule
Qwen3-32B & 16 & 1$\times$ & 71.29 \\
\midrule
KIVI-2-gs32-r0 & 2 & 8.0$\times$ & 67.90 \\
\midrule
Palu (50\%) & 4 & 8.0$\times$ & 52.02 \\
Palu (50\%) & 3 & 10.7$\times$ & 48.43 \\
\midrule
STAR-KV (50\%) & 3.3 & 9.7$\times$ & 63.30 \\
\midrule
Ours (70\%) & 16 & 3.3$\times$ & 70.45 \\
Ours (70\%) + Palu-Q & 4 & 13.3$\times$ & 70.12 \\
Ours (70\%) + Palu-Q & 3 & 17.8$\times$ & 69.36 \\
Ours (70\%) + STAR-Q & 3.2 & 16.7$\times$ & 70.28 \\
\bottomrule
\end{tabular}
\end{minipage}
\end{figure}

\begin{table}[b]
\centering
\caption{End-to-end generation throughput (tokens/s). The gray text in brackets denotes batch size. Baselines use 50\% KV cache compression and \ours{} uses 70\% with STAR-KV.}
\label{tab:throughput}
\small
\setlength{\tabcolsep}{4pt}
\begin{tabular}{lcccccc}
\toprule
\textbf{Context Len.} & \textbf{PyTorch SDPA} & \textbf{STAR-KV} & \textbf{Palu} & \textbf{ERC-SVD} & \textbf{\ours} & \textbf{Gain (\ours)} \\
\midrule
1K  & 2091 \textcolor{gray}{(128)} & 3092 \textcolor{gray}{(256)} & 2485 \textcolor{gray}{(256)} & 855 \textcolor{gray}{(256)} & 4859 \textcolor{gray}{(1024)} & 2.32$\times$ \\
2K  & 1044 \textcolor{gray}{(64)} & 1681 \textcolor{gray}{(128)} & 1340 \textcolor{gray}{(128)} & 456 \textcolor{gray}{(128)} & 2935 \textcolor{gray}{(512)} & 2.81$\times$ \\
4K  & 530 \textcolor{gray}{(32)} & 867 \textcolor{gray}{(64)} & 688 \textcolor{gray}{(64)} & 234 \textcolor{gray}{(64)} & 1642 \textcolor{gray}{(256)} & 3.10$\times$ \\
8K  & 271 \textcolor{gray}{(16)} & 437 \textcolor{gray}{(32)} & 343 \textcolor{gray}{(32)} & 117 \textcolor{gray}{(32)} & 869 \textcolor{gray}{(128)} & 3.21$\times$ \\
16K & 137 \textcolor{gray}{(8)} & 211 \textcolor{gray}{(16)} & 174 \textcolor{gray}{(16)} & 58 \textcolor{gray}{(16)} & 447 \textcolor{gray}{(64)} & 3.26$\times$ \\
\bottomrule
\end{tabular}
\end{table}

\subsection{Quantization.}
Table~\ref{tab:quantization} reports quantization results on Qwen3-32B. Applying Palu-Q and STAR-Q to \ours{} at 70\% compression yields 13.3$\times$ and 16.7$\times$ compression at 70.12\% and 70.28\% accuracy, nearly matching the unquantized model, with 3-bit Palu-Q further raising compression to 17.8$\times$ at 69.36\%. In contrast, Palu and STAR-KV reach only 48.43\% and 63.30\% at lower compression (10.7$\times$ and 9.7$\times$), and KIVI-2 attains 67.90\% at 8$\times$. The accuracy advantage of \ours{} under low-rank compression thus carries over to the quantized cache.

Other experiments, e.g., ablation studies and verifying the assumption of our theoretical analysis, are included in Appendix \ref{app: add_res}.
\section{Related Work}

\paragraph{Low-rank decomposition for general LLM compression.}
A broad line of work decomposes projection weights to reduce model memory and computation, including SVD-LLM~\citep{wang2025svd}, ASVD~\citep{yuan2023asvd}, Zero Sum SVD~\citep{abbasi2026zero}, ERC-SVD~\citep{bai2025erc}, and S2VD~\citep{zhou2026s}. These methods motivate data-weighted factorizations, but their layer- or weight-reconstruction objectives do not directly reflect attention behavior under KV caching. \ours instead derives its weighting from attention outputs and attention-distribution curvature.

\paragraph{General KV-cache compression.}
KV-cache compression has also been studied along the token and precision axes: token-selection methods retain a subset of cached tokens \citep{zhang2023h2o,liu2023scissorhands,ge2024model,li2024snapkv,cai2024pyramidkv,oren2024transformers}, while quantization methods reduce cache precision \citep{zhao2024atom,liu2024kivi,lin2025qserve,su2025rotatekv,hooper2024kvquant,he2024zipcache}. These are complementary to low-rank compression, which reduces the hidden dimension rather than the sequence dimension or bit-width. Our joint allocator further treats rank and bit-width as choices under a shared memory budget.

\section{conclusion}
We revisited low-rank KV cache compression and addressed two limitations of existing methods: objectives that minimize KV reconstruction error rather than preserve attention outputs, and rank allocation that collapses to uniform thresholds. We reformulated compression to align with generation outputs, derived a closed-form solution via an approximate objective, and proposed a global rank allocation mechanism with a hard budget that preserves the truly important components across heads. Extensive experiments show consistent accuracy and efficiency gains under aggressive compression, with up to $3.26\times$ faster decoding and 87\% lower memory at comparable accuracy.

\section*{Acknowledgments}
This was was supported in part by......

\bibliographystyle{unsrt}  
\bibliography{references}  

\newpage
\appendix

\section{Derivation of $\mathcal L_V(\widehat W_V)$ and
$\mathcal L_K(\widehat W_K)$}
\label{objective_derivation}

We derive separate value and key criteria from the single-head
attention-output error using the notation of
Section~\ref{sec:prelim:attn}.
For each calibration query $x$, $X$ denotes its corresponding
context matrix; we suppress this dependence to match the
main-text notation.
Calibration representations and queries are held fixed, and
all attention distributions are restricted to their common
finite, unmasked support.
All objectives below are summed over calibration queries,
consistently with Equation~\ref{eq:joint_kv_objective}.

For brevity, write
\[
a_x=a_x(W_K),
\qquad
\widehat a_x=a_x(\widehat W_K).
\]
Using Equation~\ref{eq: attention_out}, adding and subtracting
$a_xX\widehat W_VW_O$ gives the exact decomposition
\begin{equation}
\begin{aligned}
&o_x(W_K,W_V)-o_x(\widehat W_K,\widehat W_V)
\\
&\quad=
\underbrace{
a_xX(W_V-\widehat W_V)W_O
}_{\text{value perturbation}}
+
\underbrace{
(a_x-\widehat a_x)X\widehat W_VW_O
}_{\text{key perturbation}}.
\end{aligned}
\label{eq:joint_error_decomposition}
\end{equation}

Applying Young's inequality to these two terms and summing
over queries yields, for every $\xi>0$,
\begin{equation}
\mathcal J(\widehat W_K,\widehat W_V)
\le
(1+\xi)\mathcal L_V(\widehat W_V)
+
(1+\xi^{-1})
\mathcal L_{K,\mathrm{out}}(\widehat W_K;\widehat W_V),
\label{eq:joint_young_bound}
\end{equation}
where
\begin{align}
\mathcal L_V(\widehat W_V)
&=
\sum_x
\left\|
a_xX(W_V-\widehat W_V)W_O
\right\|_2^2,
\label{eq:value_objective}\\
\mathcal L_{K,\mathrm{out}}(\widehat W_K;\widehat W_V)
&=
\sum_x
\left\|
(a_x-\widehat a_x)X\widehat W_VW_O
\right\|_2^2.
\label{eq:key_output_objective}
\end{align}
Thus $\mathcal L_V$ is exactly the output error caused by
compressing the value projection while keeping the original
attention distributions fixed.
The key-output term remains coupled to $\widehat W_V$.

To obtain a value-independent key criterion, define
\begin{equation}
\mathcal L_K(\widehat W_K)
=
\sum_x
D_{\mathrm{KL}}
\left(
a_x(W_K)\,\middle\|\,a_x(\widehat W_K)
\right).
\label{eq:key_objective}
\end{equation}
For each query, let
\[
M_x(\widehat W_V)
=
\max_{1\le i\le n}
\|(X\widehat W_VW_O)_{i,:}\|_2.
\]
The triangle inequality gives
\[
\left\|
(a_x-\widehat a_x)X\widehat W_VW_O
\right\|_2
\le
M_x(\widehat W_V)\|a_x-\widehat a_x\|_1.
\]
Pinsker's inequality, with natural logarithms, then implies
\[
\left\|
(a_x-\widehat a_x)X\widehat W_VW_O
\right\|_2^2
\le
2M_x(\widehat W_V)^2
D_{\mathrm{KL}}(a_x\,\|\,\widehat a_x).
\]
Summing over queries yields the candidate-dependent bound
\begin{equation}
\mathcal L_{K,\mathrm{out}}(\widehat W_K;\widehat W_V)
\le
2\sum_x
M_x(\widehat W_V)^2
D_{\mathrm{KL}}(a_x\,\|\,\widehat a_x).
\label{eq:candidate-dependent-key-bound}
\end{equation}

Suppose a common finite constant $M$ satisfies
\[
M_x(\widehat W_V)\le M
\]
for every calibration query and every value approximation
in the considered candidate class.
Then
\begin{equation}
\mathcal L_{K,\mathrm{out}}(\widehat W_K;\widehat W_V)
\le
2M^2\mathcal L_K(\widehat W_K).
\label{eq:pinsker_key_bound}
\end{equation}
Taking $\xi=1$ in
Equation~\ref{eq:joint_young_bound} gives
\begin{equation}
\mathcal J(\widehat W_K,\widehat W_V)
\le
2\mathcal L_V(\widehat W_V)
+
4M^2\mathcal L_K(\widehat W_K).
\label{eq:joint_separate_bound}
\end{equation}

\paragraph{Scope of the bound.}
Equation~\ref{eq:joint_separate_bound} requires the same
finite $M$ throughout the considered candidate class;
a rank constraint alone does not ensure this condition.
Without a uniform norm bound,
Equation~\ref{eq:candidate-dependent-key-bound} still holds,
but its coefficients depend on $\widehat W_V$.
In that setting, $\mathcal L_K$ is an
upper-bound-motivated, distribution-preserving surrogate,
rather than a value-independent upper bound on the
unrestricted joint problem.

The value objective has an exact quadratic representation,
and its positive-definite regularization admits the
weighted-SVD solution in Theorem~\ref{thm:value_bsvd}.
For the nonlinear key objective, we derive its local
quadratic term, apply pre-fitting regularization and a
Kronecker approximation, and solve the resulting
regularized surrogate in closed form.
These solutions exactly minimize their respective
surrogates, not necessarily the original joint objective.
The analysis concerns attention-output errors on fixed
calibration inputs, rather than a guarantee of final
autoregressive generation preservation.

\section{Theoretical analysis without RoPE}

\label{app:value_proofs}

\subsection{Theoretical analysis for $\mathcal L^\approx_V(\widehat W_V)$}\label{app: value_analysis}

\subsubsection{Proof of Lemma~\ref{lemma:lv_trans_eps}}
\label{app: value_lemma}

\begin{proof}
Let
\[
E=W_V-\widehat W_V,
\qquad
u_x=a_xX_x.
\]
Then
\[
G_V^0=\sum_xu_x^\top u_x,
\qquad
H_V^0=W_OW_O^\top
\]
are positive semidefinite.
Adding $\epsilon I$ with $\epsilon>0$ makes both matrices
positive definite, so their Cholesky factors $L_V$ and $C_V$
exist and are invertible.

\paragraph{Aggregated quadratic form.}
Using the trace identity for a squared Euclidean norm and
cyclicity of the trace,
\begin{align*}
\mathcal L_V(\widehat W_V)
&=
\sum_x\|u_xEW_O\|_2^2
\\
&=
\sum_x
\operatorname{tr}
\left(u_xEH_V^0E^\top u_x^\top\right)
\\
&=
\sum_x
\operatorname{tr}
\left(u_x^\top u_xEH_V^0E^\top\right)
\\
&=
\operatorname{tr}
\left(G_V^0EH_V^0E^\top\right).
\end{align*}

\paragraph{Regularized weighted Frobenius form.}
By the Cholesky factorizations,
\begin{align*}
\mathcal L_V^\approx(\widehat W_V)
&=
\|L_VEC_V\|_F^2
\\
&=
\operatorname{tr}
\left(L_VEC_VC_V^\top E^\top L_V^\top\right)
\\
&=
\operatorname{tr}
\left(G_VEH_VE^\top\right)
\\
&=
\operatorname{tr}
\left(
(G_V^0+\epsilon I_d)
E
(H_V^0+\epsilon I_{d_h})
E^\top
\right).
\end{align*}
Expanding yields
\begin{align*}
\mathcal L_V^\approx(\widehat W_V)
&=
\operatorname{tr}(G_V^0EH_V^0E^\top)
+
\epsilon\operatorname{tr}(G_V^0EE^\top)
\\
&\quad+
\epsilon\operatorname{tr}(EH_V^0E^\top)
+
\epsilon^2\|E\|_F^2.
\end{align*}
The first term is $\mathcal L_V(\widehat W_V)$.
Furthermore,
\[
\operatorname{tr}(G_V^0EE^\top)
=
\sum_x\|u_xE\|_2^2
=
\sum_x\|a_xX_xE\|_2^2,
\]
and
\[
\operatorname{tr}(EH_V^0E^\top)
=
\|EW_O\|_F^2.
\]
Therefore,
\[
\mathcal L_V^\approx(\widehat W_V)
=
\mathcal L_V(\widehat W_V)
+
\Delta_{\epsilon,V}(E),
\]
with $\Delta_{\epsilon,V}(E) = \epsilon\sum_x\|a_xX_xE\|_2^2
+
\epsilon\|EW_O\|_F^2
+
\epsilon^2\|E\|_F^2.$
Every term in this correction is nonnegative.

Finally,
\begin{equation}
0\le\Delta_{\epsilon,V}(E)
\le
\left[
\epsilon\left(
\|G_V^0\|_{\mathrm{op}}
+
\|H_V^0\|_{\mathrm{op}}
\right)
+
\epsilon^2
\right]\|E\|_F^2.
\label{eq:value-regularization-bound}
\end{equation}
For fixed $E$ and calibration statistics, this is
$O(\epsilon)$ as $\epsilon\to0$, completing the proof.
\end{proof}

\subsubsection{Proof of Theorem~\ref{thm:value_bsvd}}\label{app: value_theorem}

\begin{proof}
By Lemma~\ref{lemma:lv_trans_eps}, the rank-constrained value problem is
\[
\min_{\operatorname{rank}(\widehat W_V)\le r}
\left\|
L_VW_VC_V-L_V\widehat W_VC_V
\right\|_F^2.
\]
Define
\[
M_V:=L_VW_VC_V,
\qquad
Y:=L_V\widehat W_VC_V.
\]
Since $L_V$ and $C_V$ are invertible,
$\operatorname{rank}(Y)=\operatorname{rank}(\widehat W_V)$, and the
transformation is bijective. The problem is therefore equivalent to
\[
\min_{\operatorname{rank}(Y)\le r}\|M_V-Y\|_F^2.
\]
By the Eckart--Young--Mirsky theorem \cite{eckart1936approximation, mirsky1960symmetric}, one global minimizer is
\[
Y_r^*
=
\bar U_{V,r}\bar\Sigma_{V,r}\bar R_{V,r}^\top,
\]
and the minimum value is
$\sum_{j=r+1}^{m}\bar\sigma_{V,j}^2$. Applying the inverse
transformation gives
\[
\widehat W_{V,r}^*
=
L_V^{-1}
\bar U_{V,r}\bar\Sigma_{V,r}\bar R_{V,r}^\top
C_V^{-1}.
\]
\end{proof}

\subsection{Theoretical analysis for $\mathcal L^\approx_K(\widehat W_K)$}\label{app: key_analysis}

\subsubsection{Proof of Theorem \ref{thm:lk_kron_rank1}}\label{app: lk_kron_rank1_proof}

Before proving Theorem \ref{thm:lk_kron_rank1}, we prove the following lemma for the objective $\mathcal{L}_K(W_K)$ evaluated on individual input token, i.e., $\mathcal{L}_{K,x}(\widehat{W}_K) = D_{\mathrm{KL}}\!\left(
 a_{x}(W_K)\,\middle\|\,a_{x}(\widehat{W}_K)
\right)$ such that $\mathcal{L}_K(\widehat{W}_K) = \sum_x \mathcal{L}_{K,x}(\widehat{W}_K)$.

\begin{lemma}[Per-query key-side quadratic approximation]
\label{lemma:lk_frob_per_token}
For a query $x$, let $J_x,A_x,G^\epsilon_{K,x}$ be defined
in \eqref{eq: lk_intermediate_res}, and write
\[
G_{K,x}=G^\epsilon_{K,x}-\epsilon I_d,
\qquad
H_{K,x}=q_xq_x^\top,
\qquad
H^\epsilon_{K,x}=H_{K,x}+\epsilon I_{d_h}.
\]
Let
\[
G^\epsilon_{K,x}=L_{K,x}^\top L_{K,x},
\qquad
H^\epsilon_{K,x}=C_{K,x}C_{K,x}^\top.
\] be the Choleskey decomposition of $G^\epsilon_{K,x}$ and $H^\epsilon_{K,x}$. 
For $E=W_K-\widehat W_K$, define
\[
\mathcal L^\approx_{K,x}(\widehat W_K)
=
\|L_{K,x}EC_{K,x}\|_F^2.
\]
Then
\begin{equation}
\mathcal L_{K,x}(\widehat W_K)
-
\mathcal L^\approx_{K,x}(\widehat W_K)
=
\mathcal R_x(E)-\Delta_{\epsilon,x}(E),
\label{eq:per-query-key-error-identity}
\end{equation}
where
\[
|\mathcal R_x(E)|
\le
\frac{1}{6d_h^{3/2}}\omega(X_xEq_x)^3,
\qquad
\omega(v):=\max_n v_n-\min_n v_n,
\]
and
\[
\Delta_{\epsilon,x}(E)
=
\epsilon\operatorname{tr}(G_{K,x}EE^\top)
+
\epsilon\|Eq_x\|_2^2
+
\epsilon^2\|E\|_F^2
\ge0.
\]
Consequently,
\begin{equation}
\left|
\mathcal L_{K,x}(\widehat W_K)
-
\mathcal L^\approx_{K,x}(\widehat W_K)
\right|
\le
\frac{1}{6d_h^{3/2}}\omega(X_xEq_x)^3
+
\Delta_{\epsilon,x}(E).
\label{eq:per-query-key-error-bound}
\end{equation}
For the one-sided regularization used in Kronecker fitting,
\begin{equation}
\mathcal L_{K,x}(\widehat W_K)
=
q_x^\top E^\top G^\epsilon_{K,x}Eq_x
+
\mathcal R_x(E)
-
\epsilon\|Eq_x\|_2^2.
\label{eq: l_k_other_form}
\end{equation}
\end{lemma}

\begin{proof}
All attention quantities are restricted to the finite,
unmasked support.
Let
\[
h_x
=
z_x(\widehat W_K)^\top-z_x(W_K)^\top
=
-\frac{1}{\sqrt{d_h}}X_xEq_x
\]
be the logit-perturbation column vector, and define
\[
F_x(t)
=
\log\left(\sum_n a_{x,n}e^{t h_{x,n}}\right),
\qquad t\in[0,1].
\]
The exact softmax KL identity gives
\[
\mathcal L_{K,x}(\widehat W_K)
=
F_x(1)-F_x(0)-F_x'(0),
\]
where $F_x(0)=0$ and $F_x'(0)=a_xh_x$.

Define the tilted probabilities and their mean by
\[
a_{x,n}^{(t)}
=
\frac{a_{x,n}e^{t h_{x,n}}}
     {\sum_j a_{x,j}e^{t h_{x,j}}},
\qquad
\mu_x(t)=\sum_n a_{x,n}^{(t)}h_{x,n}.
\]
Differentiation yields
\[
F_x''(0)=h_x^\top J_xh_x,
\qquad
F_x'''(t)
=
\sum_n a_{x,n}^{(t)}
\bigl(h_{x,n}-\mu_x(t)\bigr)^3.
\]
Taylor's theorem with integral remainder therefore gives
\begin{equation}
\mathcal L_{K,x}(\widehat W_K)
=
\frac12h_x^\top J_xh_x+\mathcal R_x(E),
\qquad
\mathcal R_x(E)
=
\frac12\int_0^1(1-t)^2F_x'''(t)\,dt.
\label{eq:per-query-logit-expansion}
\end{equation}
Since $\mu_x(t)$ lies between the minimum and maximum
entries of $h_x$,
\[
|F_x'''(t)|
\le
\sum_n a_{x,n}^{(t)}
|h_{x,n}-\mu_x(t)|^3
\le
\omega(h_x)^3.
\]
Hence
\[
|\mathcal R_x(E)|
\le
\frac16\omega(h_x)^3
=
\frac{1}{6d_h^{3/2}}\omega(X_xEq_x)^3.
\]

To connect this expansion to the previously defined metric,
write $D_x=\operatorname{diag}(a_x)$.
Because $D_x^{-1}a_x^\top=\boldsymbol 1$ and
$J_x\boldsymbol 1=0$,
\[
A_x
=
J_xD_x^{-1}J_x
=
J_xD_x^{-1}(D_x-a_x^\top a_x)
=
J_x.
\]
Substituting $h_x$ into
\eqref{eq:per-query-logit-expansion} therefore yields
\begin{align}
\mathcal L_{K,x}(\widehat W_K)
&=
q_x^\top E^\top G_{K,x}Eq_x+\mathcal R_x(E)
\nonumber\\
&=
\operatorname{tr}(G_{K,x}EH_{K,x}E^\top)
+\mathcal R_x(E).
\label{eq: l_k_inter_0}
\end{align}

By the Cholesky factorizations and cyclicity of the trace,
\begin{align*}
\mathcal L^\approx_{K,x}(\widehat W_K)
&=
\operatorname{tr}
\left(G^\epsilon_{K,x}EH^\epsilon_{K,x}E^\top\right)
\\
&=
\operatorname{tr}(G_{K,x}EH_{K,x}E^\top)
+
\epsilon\operatorname{tr}(G_{K,x}EE^\top)
\\
&\quad+
\epsilon\|Eq_x\|_2^2
+
\epsilon^2\|E\|_F^2.
\end{align*}
The additional terms equal $\Delta_{\epsilon,x}(E)$
and are nonnegative.
Subtracting this expression from \eqref{eq: l_k_inter_0}
proves \eqref{eq:per-query-key-error-identity};
the triangle inequality gives
\eqref{eq:per-query-key-error-bound}.

Finally,
\[
q_x^\top E^\top G^\epsilon_{K,x}Eq_x
=
q_x^\top E^\top G_{K,x}Eq_x
+
\epsilon\|Eq_x\|_2^2,
\]
which, together with \eqref{eq: l_k_inter_0}, proves
\eqref{eq: l_k_other_form}.
\end{proof}

\begin{lemma}[Kronecker--vectorization identity]
\label{lemma:kronecker_vec}
Let \(A\in\mathbb R^{m\times m}\) and \(B\in\mathbb R^{n\times n}\) be symmetric matrices, and let \(E\in\mathbb R^{n\times m}\). Then
\[
\boxed{
\operatorname{vec}(E)^\top(A\otimes B)\operatorname{vec}(E)
=
\operatorname{tr}\!\left(B E A E^\top\right).
}
\]
In particular, for any vector \(q\in\mathbb R^{m}\) and any symmetric matrix \(G\in\mathbb R^{n\times n}\),
\[
\boxed{
q^\top E^\top G E q
=
\operatorname{vec}(E)^\top(qq^\top\otimes G)\operatorname{vec}(E).
}
\]
\end{lemma}

\begin{proof}
We prove the general identity first.

\paragraph{Step 1: Vectorization identity.}
For matrices \(P\), \(X\), \(Q\) of compatible dimensions,
\[
\operatorname{vec}(PXQ)=(Q^\top\otimes P)\operatorname{vec}(X).
\]
Applying this with \(P=B\), \(X=E\), \(Q=A\), and using \(A=A^\top\) (since \(A\) is symmetric),
\[
(A\otimes B)\operatorname{vec}(E)
=
\operatorname{vec}(B E A).
\]

\paragraph{Step 2: Trace inner product.}
For any matrices \(P,Q\) of the same size,
\[
\operatorname{vec}(P)^\top\operatorname{vec}(Q)=\operatorname{tr}(P^\top Q).
\]
Applying this with \(P=E\) and \(Q=B E A\),
\[
\operatorname{vec}(E)^\top\operatorname{vec}(B E A)
=
\operatorname{tr}\!\left(E^\top B E A\right).
\]
Combining with Step 1,
\[
\operatorname{vec}(E)^\top(A\otimes B)\operatorname{vec}(E)
=
\operatorname{tr}\!\left(E^\top B E A\right).
\]

\paragraph{Step 3: Cyclicity of the trace.}
Using cyclicity of the trace,
\[
\operatorname{tr}\!\left(E^\top B E A\right)
=
\operatorname{tr}\!\left(B E A E^\top\right).
\]
Therefore
\[
\operatorname{vec}(E)^\top(A\otimes B)\operatorname{vec}(E)
=
\operatorname{tr}\!\left(B E A E^\top\right).
\]

\paragraph{Corollary: per-token identity.}
Set \(A=qq^\top\) and \(B=G\). Since \(qq^\top\) and \(G\) are symmetric, the general identity gives
\[
\operatorname{vec}(E)^\top(qq^\top\otimes G)\operatorname{vec}(E)
=
\operatorname{tr}\!\left(G E qq^\top E^\top\right).
\]
On the other hand, \(q^\top E^\top G E q\) is a scalar, so it equals its own trace:
\[
q^\top E^\top G E q
=
\operatorname{tr}\!\left(q^\top E^\top G E q\right)
=
\operatorname{tr}\!\left(E^\top G E qq^\top\right)
=
\operatorname{tr}\!\left(G E qq^\top E^\top\right),
\]
where the last equality follows from cyclicity of the trace. Therefore
\[
q^\top E^\top G E q
=
\operatorname{vec}(E)^\top(qq^\top\otimes G)\operatorname{vec}(E).
\]
\end{proof}

Then we formally provide the proof of Theorem \ref{thm:lk_kron_rank1} below:

\begin{proof}
Let $E=W_K-\widehat W_K$ and $e=\operatorname{vec}(E)$.
By \eqref{eq: l_k_other_form},
\[
\mathcal L_{K,x}(\widehat W_K)
=
q_x^\top E^\top G^\epsilon_{K,x}Eq_x
+
\mathcal R_x(E)
-
\epsilon\|Eq_x\|_2^2.
\]
Applying Lemma~\ref{lemma:kronecker_vec} and summing over
calibration queries gives
\begin{equation}
\mathcal L_K(\widehat W_K)
=
e^\top M_Ke
+
\sum_x\mathcal R_x(E)
-
\Delta_{\epsilon,\mathrm{pre}}(E),
\label{eq:key-prefitting-identity}
\end{equation}
where
\[
M_K=\sum_x(q_xq_x^\top)\otimes G^\epsilon_{K,x},
\qquad
\Delta_{\epsilon,\mathrm{pre}}(E)
=
\epsilon\sum_x\|Eq_x\|_2^2.
\]
Lemma~\ref{lemma:lk_frob_per_token} also gives
\[
\left|\sum_x\mathcal R_x(E)\right|
\le
\frac{1}{6d_h^{3/2}}
\sum_x\omega(X_xEq_x)^3.
\]

Let $A_K,B_K$ be the fixed positive-semidefinite
Kronecker factors specified in the theorem.
Their post-fitting regularizations are positive definite,
so $L_K$ and $C_K$ exist and are invertible.
Using their Cholesky factorizations and
Lemma~\ref{lemma:kronecker_vec},
\begin{align*}
\mathcal L_K^\approx(\widehat W_K)
&=
\|L_KEC_K\|_F^2
\\
&=
\operatorname{tr}
\left(
(B_K+\epsilon I_d)
E
(A_K+\epsilon I_{d_h})
E^\top
\right)
\\
&=
e^\top(A_K\otimes B_K)e
+
\Delta_{\epsilon,\mathrm{post}}(E),
\end{align*}
where
\[
\Delta_{\epsilon,\mathrm{post}}(E)
=
\epsilon\operatorname{tr}(B_KEE^\top)
+
\epsilon\operatorname{tr}(EA_KE^\top)
+
\epsilon^2\|E\|_F^2.
\]
Both regularization contributions are nonnegative.
Writing
\[
\Delta_{\epsilon,K}(E)
=
\Delta_{\epsilon,\mathrm{pre}}(E)
+
\Delta_{\epsilon,\mathrm{post}}(E),
\]
and subtracting the surrogate from
\eqref{eq:key-prefitting-identity} yields the exact
signed decomposition
\begin{equation}
\begin{aligned}
\mathcal L_K(\widehat W_K)
-
\mathcal L_K^\approx(\widehat W_K)
&=
e^\top(M_K-A_K\otimes B_K)e
\\
&\quad+
\sum_x\mathcal R_x(E)
-
\Delta_{\epsilon,K}(E).
\end{aligned}
\label{eq:key-exact-error-decomposition}
\end{equation}
Taking absolute values gives
\begin{equation}
\begin{aligned}
\left|
\mathcal L_K(\widehat W_K)
-
\mathcal L_K^\approx(\widehat W_K)
\right|
&\le
\left|e^\top(M_K-A_K\otimes B_K)e\right|
\\
&\quad+
\frac{1}{6d_h^{3/2}}
\sum_x\omega(X_xEq_x)^3
+
\Delta_{\epsilon,K}(E).
\end{aligned}
\label{eq:key-total-error-bound}
\end{equation}

The three terms in \eqref{eq:key-total-error-bound}
correspond to distinct approximation sources.
The first,
$\left|e^\top(M_K-A_K\otimes B_K)e\right|$,
is the absolute Kronecker approximation error evaluated
along the residual $E$.
The second is an explicit upper bound on the total
Taylor expansion error
$\left|\sum_x\mathcal R_x(E)\right|$.
The third, $\Delta_{\epsilon,K}(E)$, accounts for both
the pre-fitting ridge and the post-fitting factor
regularization. In particular,
\[
\Delta_{\epsilon,K}(E)
\le
\epsilon\sum_x\|Eq_x\|_2^2
+
\left[
\epsilon\bigl(\|A_K\|_{\mathrm{op}}+\|B_K\|_{\mathrm{op}}\bigr)
+\epsilon^2
\right]\|E\|_F^2.
\]
For fixed $E$ and calibration data, this term is
$O(\epsilon)$ as $\epsilon\to0$, provided the fitted
factors remain bounded in operator norm.
Without this boundedness condition, the explicit
regularization term should be retained rather than
abbreviated as $O(\epsilon)$.

\end{proof}

\subsubsection{Corollary for decomposing key tensor}\label{app: corollory_key}
\begin{corollary}   
\label{thm:key_bsvd}
let $L_KW_KC_K = \bar U_K\bar\Sigma_K\bar R_K^\top$ be a thin singular value decomposition, where $\bar\Sigma_K =
\operatorname{diag}
(\bar\sigma_{K,1},\ldots,\bar\sigma_{K,m}),
m=\min(d,d_h),
$
and
$\bar\sigma_{K,1}\ge\cdots\ge\bar\sigma_{K,m}\ge0$.
Let $\bar U_{K,1:r}$, $\bar\Sigma_{K,1:r}$, and $\bar R_{K,1:r}$ denote the
components associated with the largest $r$ singular values. Then one
global minimizer of $\min\nolimits_{\operatorname{rank}(\widehat W_K)\le r} \mathcal L^{\approx}_K(\widehat W_K)$
is
\begin{small}
\begin{equation}
\widehat W_{K,1:r}^*
=
L_K^{-1}
\bar U_{K,1:r}\bar\Sigma_{K,1:r}\bar R_{K,1:r}^\top
C_K^{-1}.
\label{eq:value_bsvd_solution}
\end{equation}
\end{small}

The corresponding projection matrix, $\Wdown_{K,1:r}
=
L_V^{-1}\bar U_{K,1:r}\bar\Sigma_{K,1:r}^{1/2}$, and $\Wup_{K,1:r}
=
\bar\Sigma_{K,1:r}^{1/2}\bar R_{K,1:r}^\top C_K^{-1}.
$
\end{corollary}

\subsubsection{Overall algorithm for computing $\mathcal{L}_K^{\approx}(W_K)$}\label{app: rank_1_approx_algo}
\begin{algorithm}[t]
\caption{Rank-1 Kronecker-factored low-rank approximation of \(W_K\)}
\label{alg:lk_kron_rank1}
\begin{algorithmic}[1]
\STATE \textbf{Input:} KV cache \(X\), query vectors \(\{q_x\}\), current key matrix \(W_K\), rank \(r\), regularization \(\epsilon\).
\FOR{each query token \(x\)}
    \STATE Compute attention distribution: \(p_x\leftarrow\operatorname{Softmax}\!\bigl(\tfrac{1}{\sqrt{d_h}}XW_Kq_x\bigr)\).
    \STATE Compute softmax Jacobian: \(J_x\leftarrow\operatorname{diag}(p_x)-p_xp_x^\top\).
    \STATE Compute Fisher weight: \(A_x\leftarrow J_x\operatorname{diag}(1/p_x)J_x\).
    \STATE Compute per-token left weight: \(G_{K,x}^{\epsilon}\leftarrow\tfrac{1}{2d_h}X^\top A_xX+\epsilon I_d\).
\ENDFOR
\STATE Form exact Kronecker quadratic form: \(M_K\leftarrow\sum_x q_xq_x^\top\otimes G_{K,x}^{\epsilon}\).
\STATE Compute rank-1 Kronecker approximation \(M_K\approx A_K\otimes B_K\) via alternating least squares.
\STATE Regularize: \(B_K^{\epsilon}\leftarrow B_K+\epsilon I_d\), \(A_K^{\epsilon}\leftarrow A_K+\epsilon I_{d_h}\).
\STATE Cholesky: \(B_K^{\epsilon}=L_K^\top L_K\), \(A_K^{\epsilon}=C_KC_K^\top\).
\STATE Compute SVD \(L_KW_KC_K=U\Sigma V^\top\).
\STATE \textbf{Return} \(\widehat W_K^\star\leftarrow L_K^{-1}U_r\Sigma_rV_r^\top C_K^{-1}\).
\end{algorithmic}
\end{algorithm}
Algorithm~\ref{alg:lk_kron_rank1} summarizes the rank-1 Kronecker-factored low-rank approximation of the key projection matrix \(W_K\). Given the KV cache \(X\), the query vectors \(\{q_x\}\), and the current key matrix \(W_K\), the algorithm aims to find a rank-\(r\) approximation \(\widehat W_K\) that minimizes a quadratic surrogate of the token-summed KL divergence between the original and approximated attention distributions. The algorithm proceeds in three main stages. First, for each query token \(x\), it computes the attention distribution \(p_x\), the softmax Jacobian \(J_x\), the Fisher weight \(A_x\), and the regularized per-token left weight \(G_{K,x}^{\epsilon}\). Second, it assembles the exact Kronecker quadratic form \(M_K = \sum_x q_x q_x^\top \otimes G_{K,x}^{\epsilon}\) and approximates it by a rank-1 Kronecker product \(A_K \otimes B_K\) using alternating least squares. Third, it regularizes and Cholesky-factorizes \(A_K\) and \(B_K\) to obtain \(L_K\) and \(C_K\), and then solves the resulting weighted low-rank approximation problem in closed form via truncated SVD of \(L_K W_K C_K\). The output is the optimal rank-\(r\) key matrix \(\widehat W_K^\star\). The alternating least squares step avoids forming the full \((d_h d) \times (d_h d)\) matrix \(M_K\), making the algorithm scalable.

\section{Extension to Grouped-Query Attention}
\label{app:gqa}

Grouped-query attention (GQA) shares one key projection and one
value projection among multiple query heads.
Consequently, compression must assign a single pair of low-rank
projections to each KV group, while accounting for the behavior
of all query heads in that group.

We first consider one KV group with query-head index set
$\mathcal H$ and size $H=|\mathcal H|$.
The group shares
\[
W_K,W_V\in\mathbb R^{d\times d_h}.
\]
Each query head $h\in\mathcal H$ has its own query
$q_{x,h}\in\mathbb R^{d_h}$ and output-projection block
$O_h\in\mathbb R^{d_h\times d}$.
For calibration query $x$, let
$X_x\in\mathbb R^{n_x\times d}$ contain its causal context
representations, and let
$p_{x,h}(W_K)\in\mathbb R^{n_x}$ denote the corresponding
attention probability column vector.
All attention distributions are restricted to the finite,
unmasked support.

The contribution of this group to the layer output is
\begin{equation}
o_x(W_K,W_V)
=
\sum_{h\in\mathcal H}
p_{x,h}(W_K)^\top X_xW_VO_h.
\label{eq:gqa-group-output}
\end{equation}
All derivations below condition on the calibration inputs
and query vectors.

\subsection{Output-Error Decomposition}
\label{app:gqa-output-bound}

Let
\[
E_V=W_V-\widehat W_V,
\qquad
p_{x,h}=p_{x,h}(W_K),
\qquad
\widehat p_{x,h}=p_{x,h}(\widehat W_K).
\]
Adding and subtracting the output obtained with original
attention and compressed values yields
\begin{align}
o_x(W_K,W_V)-o_x(\widehat W_K,\widehat W_V)
&=
\underbrace{
\sum_{h\in\mathcal H}
p_{x,h}^\top X_xE_VO_h
}_{v_x}
\nonumber\\
&\quad+
\underbrace{
\sum_{h\in\mathcal H}
(p_{x,h}-\widehat p_{x,h})^\top
X_x\widehat W_VO_h
}_{k_x}.
\label{eq:gqa-output-decomposition}
\end{align}
Define
\begin{align}
\mathcal J_{\mathrm{GQA}}
&=
\sum_x
\|o_x(W_K,W_V)
-o_x(\widehat W_K,\widehat W_V)\|_2^2,
\\
\mathcal L_{V,\mathrm{GQA}}(\widehat W_V)
&=
\sum_x\|v_x\|_2^2,
\label{eq:gqa-value-objective}
\\
\mathcal L_{K,\mathrm{GQA}}(\widehat W_K)
&=
\sum_x\sum_{h\in\mathcal H}
D_{\mathrm{KL}}(p_{x,h}\,\|\,\widehat p_{x,h}).
\label{eq:gqa-key-objective}
\end{align}

Suppose that, on the considered candidate class,
\[
\max_{x,h,n}
\|(X_x\widehat W_VO_h)_{n,:}\|_2
\le M
\]
for a common finite constant $M$.
Then
\begin{equation}
\mathcal J_{\mathrm{GQA}}
\le
2\mathcal L_{V,\mathrm{GQA}}
+
4HM^2\mathcal L_{K,\mathrm{GQA}}.
\label{eq:gqa-output-bound}
\end{equation}

To verify this, write
\[
k_{x,h}
=
(p_{x,h}-\widehat p_{x,h})^\top
X_x\widehat W_VO_h.
\]
The row-norm bound and Pinsker's inequality imply
\[
\|k_{x,h}\|_2^2
\le
M^2\|p_{x,h}-\widehat p_{x,h}\|_1^2
\le
2M^2D_{\mathrm{KL}}(p_{x,h}\,\|\,\widehat p_{x,h}).
\]
Using
$\|\sum_h k_{x,h}\|_2^2\le H\sum_h\|k_{x,h}\|_2^2$
and
$\|v_x+k_x\|_2^2\le2\|v_x\|_2^2+2\|k_x\|_2^2$
gives \eqref{eq:gqa-output-bound}.

As in the single-head case, this is a conditional upper bound:
a rank constraint alone does not provide the required uniform
row-norm bound.

\subsection{Shared-Key Compression}
\label{app:gqa-keys}

The shared key projection affects every query head in the group.
Accordingly, the key objective
\eqref{eq:gqa-key-objective} aggregates their attention KL
divergences before constructing the low-rank approximation.

Let
\[
E_K=W_K-\widehat W_K,
\qquad
e_K=\operatorname{vec}(E_K).
\]
For each query-head pair $(x,h)$, define
\[
J_{x,h}
=
\operatorname{diag}(p_{x,h})-p_{x,h}p_{x,h}^\top.
\]
With RoPE, define the logit-perturbation map
\[
(\mathcal B_{x,h})_{n,:}
=
\operatorname{vec}
\left(
X_{x,n}^\top q_{x,h}^\top R_{n-m_x}
\right)^\top.
\]
The corresponding logit perturbation is
\[
\Delta z_{x,h}
=
-\frac{1}{\sqrt{d_h}}\mathcal B_{x,h}e_K.
\]
Applying the logit-space KL expansion separately to each
query head gives
\begin{equation}
\mathcal L_{K,\mathrm{GQA}}(\widehat W_K)
=
e_K^\top M_{K,\mathrm{GQA}}e_K
+
\sum_{x,h}\mathcal R_{x,h},
\label{eq:gqa-key-expansion}
\end{equation}
where
\begin{equation}
M_{K,\mathrm{GQA}}
=
\frac{1}{2d_h}
\sum_x\sum_{h\in\mathcal H}
\mathcal B_{x,h}^\top J_{x,h}\mathcal B_{x,h}
\succeq0.
\label{eq:gqa-key-curvature}
\end{equation}
Using $\omega(v)=\max_n v_n-\min_n v_n$, the remainder obeys
\begin{equation}
|\mathcal R_{x,h}|
\le
\frac{1}{6d_h^{3/2}}
\omega(\mathcal B_{x,h}e_K)^3.
\label{eq:gqa-key-remainder}
\end{equation}

Without RoPE,
$\mathcal B_{x,h}=q_{x,h}^\top\otimes X_x$, and therefore
\[
M_{K,\mathrm{GQA}}
=
\frac{1}{2d_h}
\sum_{x,h}
(q_{x,h}q_{x,h}^\top)
\otimes
(X_x^\top J_{x,h}X_x).
\]
Thus the shared-key extension consists of aggregating the
curvature of all associated query heads.
The Kronecker fitting and low-rank optimization are performed
once for the shared key projection, rather than independently
for each query head.

\subsection{Shared-Value Compression}
\label{app:gqa-values}

Sharing values requires additional care because the group output
sums the contributions of multiple query heads before its error
is measured.
Define
\[
u_{x,h}=p_{x,h}^\top X_x\in\mathbb R^{1\times d},
\qquad
e_V=\operatorname{vec}(E_V).
\]
Then
\[
v_x=\sum_{h\in\mathcal H}u_{x,h}E_VO_h.
\]
By column vectorization,
\[
v_x^\top=T_xe_V,
\qquad
T_x
=
\sum_{h\in\mathcal H}
(O_h^\top\otimes u_{x,h})
\in\mathbb R^{d\times dd_h}.
\]
Consequently, the exact shared-value objective is
\begin{equation}
\mathcal L_{V,\mathrm{GQA}}(\widehat W_V)
=
e_V^\top M_{V,\mathrm{GQA}}e_V,
\qquad
M_{V,\mathrm{GQA}}
=
\sum_x T_x^\top T_x
\succeq0.
\label{eq:gqa-value-curvature}
\end{equation}
Expanding the metric gives
\begin{equation}
M_{V,\mathrm{GQA}}
=
\sum_x
\sum_{h,k\in\mathcal H}
(O_hO_k^\top)
\otimes
(u_{x,h}^\top u_{x,k}).
\label{eq:gqa-value-cross-head}
\end{equation}
The terms with $h\ne k$ capture cross-head interactions in the
summed group output.

In general, \eqref{eq:gqa-value-curvature} is not a single
Kronecker product.
Therefore, the single-head weighted-SVD theorem does not apply
unchanged to the exact shared-value objective.
A tractable extension approximates
$M_{V,\mathrm{GQA}}$ by a positive-semidefinite Kronecker
product before applying weighted SVD.

For comparison, a headwise surrogate that discards output
cross terms would be
\[
\mathcal L_{V,\mathrm{head}}
=
\sum_{x,h}\|u_{x,h}E_VO_h\|_2^2,
\]
with metric
\[
M_{V,\mathrm{head}}
=
\sum_{x,h}
(O_hO_h^\top)\otimes(u_{x,h}^\top u_{x,h}).
\]
It satisfies
\[
\mathcal L_{V,\mathrm{GQA}}
\le H\mathcal L_{V,\mathrm{head}},
\]
but remains a sum of Kronecker products in general.
Thus even this alternative does not automatically reduce to
one exact two-sided weighted Frobenius objective.

When $H=1$, the exact value metric reduces to
\[
M_{V,\mathrm{GQA}}
=
(O_1O_1^\top)\otimes
\left(\sum_x u_{x,1}^\top u_{x,1}\right),
\]
recovering the single-head formulation.

\subsection{Common Kronecker Surrogate and Closed-Form Solution}
\label{app:gqa-svd}

For $S\in\{K,V\}$, let $M_{S,\mathrm{GQA}}$ denote the
corresponding metric above.
Fit fixed positive-semidefinite factors
\[
M_{S,\mathrm{GQA}}
\approx A_S\otimes B_S,
\qquad
A_S\in\mathbb R^{d_h\times d_h},
\quad
B_S\in\mathbb R^{d\times d}.
\]
For $\epsilon>0$, factor
\[
B_S+\epsilon I_d=L_S^\top L_S,
\qquad
A_S+\epsilon I_{d_h}=C_SC_S^\top.
\]
The surrogate is
\[
\widehat{\mathcal L}_{S,\mathrm{GQA}}(\widehat W_S)
=
\|L_S(W_S-\widehat W_S)C_S\|_F^2.
\]
If
\[
L_SW_SC_S=U_S\Sigma_SV_S^\top,
\]
then a global rank-$r_S$ minimizer is
\begin{equation}
\widehat W_{S,r_S}
=
L_S^{-1}
U_{S,r_S}\Sigma_{S,r_S}V_{S,r_S}^\top
C_S^{-1}.
\label{eq:gqa-svd-solution}
\end{equation}
This follows because the invertible change of variables
$Y=L_S\widehat W_SC_S$ preserves rank, reducing the problem
to ordinary truncated SVD.

For
$E_S=W_S-\widehat W_S$ and
$e_S=\operatorname{vec}(E_S)$, define
\[
\Delta_{\epsilon,S}(E_S)
=
\epsilon\operatorname{tr}(B_SE_SE_S^\top)
+
\epsilon\operatorname{tr}(E_SA_SE_S^\top)
+
\epsilon^2\|E_S\|_F^2.
\]
The value approximation satisfies
\begin{equation}
\begin{aligned}
\left|
\mathcal L_{V,\mathrm{GQA}}
-
\widehat{\mathcal L}_{V,\mathrm{GQA}}
\right|
&\le
\left|
e_V^\top
(M_{V,\mathrm{GQA}}-A_V\otimes B_V)e_V
\right|
\\
&\quad+\Delta_{\epsilon,V}(E_V).
\end{aligned}
\label{eq:gqa-value-error-bound}
\end{equation}
The key approximation satisfies
\begin{equation}
\begin{aligned}
\left|
\mathcal L_{K,\mathrm{GQA}}
-
\widehat{\mathcal L}_{K,\mathrm{GQA}}
\right|
&\le
\left|
e_K^\top
(M_{K,\mathrm{GQA}}-A_K\otimes B_K)e_K
\right|
\\
&\quad+
\frac{1}{6d_h^{3/2}}
\sum_{x,h}\omega(\mathcal B_{x,h}e_K)^3
+
\Delta_{\epsilon,K}(E_K).
\end{aligned}
\label{eq:gqa-key-error-bound}
\end{equation}
Thus the GQA value construction incurs Kronecker and
regularization errors, while the key construction additionally
incurs the KL Taylor remainder.
Equation~\eqref{eq:gqa-svd-solution} exactly minimizes the
constructed surrogate, not necessarily the original grouped
objective.

\subsection{Multiple Groups and Cache Accounting}
\label{app:gqa-budget}

The construction is applied separately to each KV group.
If a layer has $G$ groups and their output errors are
$\Delta o_{x,1},\ldots,\Delta o_{x,G}$, then
\[
\left\|\sum_{g=1}^{G}\Delta o_{x,g}\right\|_2^2
\le
G\sum_{g=1}^{G}\|\Delta o_{x,g}\|_2^2.
\]
Independent groupwise optimization therefore need not exactly
minimize the total layer-output error, whose expansion also
contains cross-group terms.

A shared KV group stores only one compressed key representation
and one compressed value representation per token.
If their ranks are $r_{K,g}$ and $r_{V,g}$, its cache stores
$r_{K,g}+r_{V,g}$ latent coordinates per token, regardless
of the number of query heads in the group.
Accordingly, at a common precision and context length, the
global rank budget is imposed over KV groups:
\[
\sum_{\ell,g}
\left(r_{K,\ell,g}+r_{V,\ell,g}\right)
\le B_{\mathrm{cache}},
\]
rather than counting the shared projections once per query head.

\section{Extension to RoPE}\label{app: rope_extension}

\subsection{RoPE-Aware Attention and KV Cache}
\label{sec:prelim:attn}

Rotary Position Embedding (RoPE)~\citep{su2024roformer} injects position information by rotating the query and key vectors before the attention dot product. Let \(m_x\) denote the position of the query token \(x\), and let \(n\) index the positions of the keys stored in the KV cache. For each position \(m\), let \(R_m\in\mathbb R^{d_h\times d_h}\) be the orthogonal rotation matrix associated with position \(m\). For a query row vector \(q_x=xW_Q\in\mathbb R^{1\times d_h}\) and a key row vector \(k_n=X_nW_K\in\mathbb R^{1\times d_h}\), where \(X_n\in\mathbb R^{1\times d}\) is the \(n\)-th row of \(X\), RoPE applies
\[
\tilde q_x = q_x R_{m_x}^\top,\qquad \tilde k_n = k_n R_n^\top.
\]
The attention logit between query \(x\) and key \(n\) is therefore
\[
z_{x,n}(W_K)
=
\frac{\tilde q_x \tilde k_n^\top}{\sqrt{d_h}}
=
\frac{q_x R_{m_x}^\top R_n W_K^\top X_n^\top}{\sqrt{d_h}}
=
\frac{q_x R_{n-m_x} W_K^\top X_n^\top}{\sqrt{d_h}},
\]
where the relative rotation is
\[
R_{n-m_x}:=R_{m_x}^\top R_n.
\]
The RoPE attention probability vector is
\[
p_x(W_K)
=
\operatorname{Softmax}\bigl(z_{x,1}(W_K),\dots,z_{x,n}(W_K)\bigr)
\in\mathbb R^{1\times n}.
\]
The attention output is
\[
o_x(W_K,W_V)=p_x(W_K) X W_V W_O \in\mathbb R^{1\times d},
\]
with \(p_x(W_K)\) computed with RoPE. When RoPE is disabled, \(R_{n-m_x}=I\), and \(p_x(W_K)\) reduces to the NoPE attention vector \(a_x(W_K)\).

For the key-side analysis, under \(W_K\to\widehat W_K=W_K-E\), the logit perturbation is
\[
\Delta z_{x,n}
=
-\frac{1}{\sqrt{d_h}} q_x R_{n-m_x} E^\top X_n^\top.
\]
Collecting over \(n\),
\[
\Delta z_x
=
-\frac{1}{\sqrt{d_h}} B_x \operatorname{vec}(E),
\]
where \(B_x\in\mathbb R^{n\times dd_h}\) has rows
\[
(B_x)_{n,:}
=
\operatorname{vec}\!\left(X_n^\top q_x R_{n-m_x}\right)^\top.
\]
When RoPE is disabled, \(R_{n-m_x}=I\), and the above notation reduces to the NoPE case with
\[
z_{x,n}(W_K)=\frac{1}{\sqrt{d_h}}q_xW_K^\top X_n^\top,
\qquad
B_x \text{ replaced by } (q_x\otimes X_n^\top)^\top.
\]

\subsection{Extension of Lemma~\ref{lemma:lv_trans_eps}
and Theorem~\ref{thm:value_bsvd}}

Let $a_x^{\mathrm{RoPE}}=a_x^{\mathrm{RoPE}}(W_K)$ denote
the original attention probability row vector computed with
RoPE. As in Appendix~\ref{objective_derivation}, $X$ denotes
the context associated with query $x$, with this dependence
suppressed.
The attention output is
\[
o_x(W_K,W_V)
=
a_x^{\mathrm{RoPE}}XW_VW_O.
\]
RoPE acts on queries and keys, but not on the value or output
projections. Holding the original attention distributions
fixed, the value objective is therefore
\[
\mathcal L_V(\widehat W_V)
=
\sum_x
\left\|
a_x^{\mathrm{RoPE}}X(W_V-\widehat W_V)W_O
\right\|_2^2.
\]

Writing $E=W_V-\widehat W_V$, define the aggregated statistics
\[
G_V^0
=
\sum_x
X^\top(a_x^{\mathrm{RoPE}})^\top a_x^{\mathrm{RoPE}}X,
\qquad
H_V^0=W_OW_O^\top.
\]
The same trace calculation as in
Lemma~\ref{lemma:lv_trans_eps} gives
\[
\mathcal L_V(\widehat W_V)
=
\operatorname{tr}(G_V^0EH_V^0E^\top).
\]
For $\epsilon>0$, factor
\[
G_V=G_V^0+\epsilon I_d=L_V^\top L_V,
\qquad
H_V=H_V^0+\epsilon I_{d_h}=C_VC_V^\top.
\]
The regularized surrogate satisfies
\[
\mathcal L_V^\approx(\widehat W_V)
=
\|L_VEC_V\|_F^2
=
\mathcal L_V(\widehat W_V)+\Delta_{\epsilon,V}(E),
\]
where
\[
\Delta_{\epsilon,V}(E)
=
\epsilon\sum_x\|a_x^{\mathrm{RoPE}}XE\|_2^2
+
\epsilon\|EW_O\|_F^2
+
\epsilon^2\|E\|_F^2
\ge0.
\]
Thus Lemma~\ref{lemma:lv_trans_eps} holds with the
RoPE-enabled attention distributions used to compute $G_V^0$.
For fixed $W_O$ and $\epsilon$, $H_V$ and its factor $C_V$
are unchanged.

Since $L_V$ and $C_V$ are invertible,
Theorem~\ref{thm:value_bsvd} applies to the regularized
surrogate without further modification.
If
\[
L_VW_VC_V
=
\bar U_V\bar\Sigma_V\bar R_V^\top,
\]
then a global minimizer over
$\operatorname{rank}(\widehat W_V)\le r$ is
\[
\widehat W_{V,r}^*
=
L_V^{-1}
\bar U_{V,1:r}\bar\Sigma_{V,1:r}
\bar R_{V,1:r}^\top
C_V^{-1},
\]
with factors
\[
\Wdown_{V,1:r}
=
L_V^{-1}\bar U_{V,1:r}\bar\Sigma_{V,1:r}^{1/2},
\qquad
\Wup_{V,1:r}
=
\bar\Sigma_{V,1:r}^{1/2}
\bar R_{V,1:r}^\top C_V^{-1}.
\]

Even when keys are also compressed,
$a_x^{\mathrm{RoPE}}$ in this value objective remains the
original attention distribution.
The additional output error caused by changing attention
is accounted for separately by the key-output term in
Appendix~\ref{objective_derivation}.
Accordingly, the solution above exactly minimizes the
regularized fixed-attention value surrogate, not the joint
key--value objective.

\subsection{Extension of Theorem \ref{thm:lk_kron_rank1}}

\begin{lemma}[Rank-1 Kronecker weighted Frobenius form with RoPE]
\label{lemma:lk_kron_rank1_rope}
For each calibration query $x$ at position $m_x$, let
$q_x\in\mathbb R^{d_h}$ be its unrotated query column vector,
and let $X_x\in\mathbb R^{n_x\times d}$ contain the context
representations on its finite, unmasked attention support.
Write $X_{x,n}$ for the $n$-th row of $X_x$, and define
\[
z_{x,n}(W_K)
=
\frac{1}{\sqrt{d_h}}
q_x^\top R_{n-m_x}W_K^\top X_{x,n}^\top,
\qquad
p_x=\operatorname{softmax}(z_x(W_K)).
\]
Here $z_x$ and $p_x$ are column vectors and $p_{x,n}>0$.
Define
\[
J_x=\operatorname{diag}(p_x)-p_xp_x^\top,
\qquad
(B_x)_{n,:}
=
\operatorname{vec}
\left(X_{x,n}^\top q_x^\top R_{n-m_x}\right)^\top,
\]
where $\operatorname{vec}$ stacks matrix columns.

The exact, unregularized second-order quadratic metric is
\begin{equation}
M_K^0
=
\frac{1}{2d_h}\sum_x B_x^\top J_xB_x
=
\frac{1}{2d_h}\sum_x
B_x^\top J_x\operatorname{diag}(1/p_x)J_xB_x.
\label{eq:rope-key-metric}
\end{equation}
For $\epsilon>0$, define the pre-fitting regularized metric
\begin{equation}
P_K
=
\sum_x(q_xq_x^\top)\otimes I_d,
\qquad
M_K^\epsilon=M_K^0+\epsilon P_K.
\label{eq:rope-prefitting-metric}
\end{equation}
This regularizer is defined in the unrotated query coordinates.

Let $A_K\succeq0$ and $B_K\succeq0$ be fixed factors of a
rank-one Kronecker approximation
\[
M_K^\epsilon\approx A_K\otimes B_K,
\]
with dimensions $d_h\times d_h$ and $d\times d$, respectively.
The factors may depend on $\epsilon$ through the fitting target.
Define the post-fitting regularizations
\[
A_K^\epsilon=A_K+\epsilon I_{d_h},
\qquad
B_K^\epsilon=B_K+\epsilon I_d,
\]
and factor
\[
B_K^\epsilon=L_K^\top L_K,
\qquad
A_K^\epsilon=C_KC_K^\top.
\]
Then the surrogate
\[
\widehat{\mathcal L}_K(\widehat W_K)
=
\left\|L_K(W_K-\widehat W_K)C_K\right\|_F^2
\]
has a global minimizer over
$\operatorname{rank}(\widehat W_K)\le r$ given by
\begin{equation}
\widehat W_K^\star
=
L_K^{-1}U_r\Sigma_rV_r^\top C_K^{-1},
\label{eq:rope-key-svd-solution}
\end{equation}
where $L_KW_KC_K=U\Sigma V^\top$ and
$U_r,\Sigma_r,V_r$ are its top-$r$ components.

For $E=W_K-\widehat W_K$, $e=\operatorname{vec}(E)$, and
\[
\mathcal L_{K,x}(\widehat W_K)
=
D_{\mathrm{KL}}
\left(
p_x\,\|\,\operatorname{softmax}(z_x(\widehat W_K))
\right),
\]
the approximation error satisfies
\begin{equation}
\begin{aligned}
\left|
\sum_x\mathcal L_{K,x}(\widehat W_K)
-
\widehat{\mathcal L}_K(\widehat W_K)
\right|
&\le
\underbrace{
\left|e^\top(M_K^\epsilon-A_K\otimes B_K)e\right|
}_{\text{Kronecker approximation}}
\\
&\quad+
\underbrace{
\frac{1}{6d_h^{3/2}}
\sum_x\omega(B_xe)^3
}_{\text{Taylor remainder bound}}
+
\underbrace{
\Delta_\epsilon(E)
}_{\text{total regularization}},
\end{aligned}
\label{eq:rope-key-error-bound}
\end{equation}
where
\[
\omega(v)=\max_n v_n-\min_n v_n
\]
and
\begin{align}
\Delta_\epsilon(E)
&=
\Delta_{\epsilon,\mathrm{pre}}(E)
+
\Delta_{\epsilon,\mathrm{post}}(E),
\label{eq:rope-key-regularization-error}
\\
\Delta_{\epsilon,\mathrm{pre}}(E)
&=
\epsilon e^\top P_Ke
=
\epsilon\sum_x\|Eq_x\|_2^2,
\label{eq:rope-prefitting-error}
\\
\Delta_{\epsilon,\mathrm{post}}(E)
&=
\epsilon\operatorname{tr}(B_KEE^\top)
+
\epsilon\operatorname{tr}(EA_KE^\top)
+
\epsilon^2\|E\|_F^2.
\label{eq:rope-postfitting-error}
\end{align}
Both regularization contributions are nonnegative.

When RoPE is disabled, $R_{n-m_x}=I$ and
$B_x=q_x^\top\otimes X_x$.
Consequently,
\begin{equation}
M_K^\epsilon
=
\sum_x(q_xq_x^\top)\otimes
\left(
\frac{1}{2d_h}X_x^\top J_xX_x+\epsilon I_d
\right),
\label{eq:rope-prefitting-nope-reduction}
\end{equation}
recovering the non-RoPE construction with a pre-fitting
ridge in the per-query left metric.
\end{lemma}

\begin{proof}
Below is the proof:
\paragraph{Step 1: RoPE-aware logit perturbation.}
For $E=W_K-\widehat W_K$, define
\[
h_{x,n}
=
z_{x,n}(\widehat W_K)-z_{x,n}(W_K)
=
-\frac{1}{\sqrt{d_h}}
q_x^\top R_{n-m_x}E^\top X_{x,n}^\top.
\]
Set $u=X_{x,n}^\top\in\mathbb R^d$ and
$v=R_{n-m_x}^\top q_x\in\mathbb R^{d_h}$.
The column-vectorization identity gives
\[
u^\top Ev
=
\operatorname{vec}(E)^\top(v\otimes u)
=
\operatorname{vec}(E)^\top\operatorname{vec}(uv^\top).
\]
Therefore,
\[
q_x^\top R_{n-m_x}E^\top X_{x,n}^\top
=
e^\top
\operatorname{vec}
\left(X_{x,n}^\top q_x^\top R_{n-m_x}\right).
\]
Stacking over the attention support yields
\begin{equation}
h_x=-\frac{1}{\sqrt{d_h}}B_xe.
\label{eq:rope-logit-perturbation}
\end{equation}

\paragraph{Step 2: Direct KL expansion in logit space.}
For a fixed query $x$, define
\[
F_x(t)
=
\log\left(
\sum_{n=1}^{n_x}p_{x,n}e^{t h_{x,n}}
\right),
\qquad t\in[0,1].
\]
The exact softmax KL identity gives
\begin{equation}
\mathcal L_{K,x}(\widehat W_K)
=
F_x(1)-F_x(0)-F_x'(0).
\label{eq:rope-kl-logit-identity}
\end{equation}
Here $F_x(0)=0$ and $F_x'(0)=p_x^\top h_x$.

Define
\[
p_{x,n}^{(t)}
=
\frac{p_{x,n}e^{t h_{x,n}}}
     {\sum_j p_{x,j}e^{t h_{x,j}}},
\qquad
\mu_x(t)=\sum_n p_{x,n}^{(t)}h_{x,n}.
\]
Differentiation yields
\[
F_x''(0)=h_x^\top J_xh_x,
\qquad
F_x'''(t)
=
\sum_n
p_{x,n}^{(t)}
\bigl(h_{x,n}-\mu_x(t)\bigr)^3.
\]
Taylor's theorem with integral remainder therefore gives
\begin{equation}
\mathcal L_{K,x}(\widehat W_K)
=
\frac12h_x^\top J_xh_x+\mathcal R_x,
\qquad
\mathcal R_x
=
\frac12\int_0^1(1-t)^2F_x'''(t)\,dt.
\label{eq:rope-kl-taylor}
\end{equation}
Since $\mu_x(t)$ lies between the minimum and maximum
entries of $h_x$,
\[
|F_x'''(t)|
\le
\sum_n p_{x,n}^{(t)}
|h_{x,n}-\mu_x(t)|^3
\le
\omega(h_x)^3.
\]
It follows that
\begin{equation}
|\mathcal R_x|
\le
\frac16\omega(h_x)^3
=
\frac{1}{6d_h^{3/2}}\omega(B_xe)^3.
\label{eq:rope-explicit-taylor-bound}
\end{equation}

\paragraph{Step 3: Exact metric and pre-fitting ridge.}
Substituting \eqref{eq:rope-logit-perturbation} into
\eqref{eq:rope-kl-taylor} gives
\[
\mathcal L_{K,x}(\widehat W_K)
=
\frac{1}{2d_h}e^\top B_x^\top J_xB_xe+\mathcal R_x.
\]
To verify the alternative expression for $M_K^0$, let
$D_x=\operatorname{diag}(p_x)$.
Since $D_x^{-1}p_x=\boldsymbol 1$ and
$J_x\boldsymbol 1=0$,
\[
J_xD_x^{-1}J_x
=
J_xD_x^{-1}(D_x-p_xp_x^\top)
=
J_x-J_x\boldsymbol 1p_x^\top
=
J_x.
\]
Thus the two expressions in
\eqref{eq:rope-key-metric} coincide, and
\begin{equation}
\sum_x\mathcal L_{K,x}(\widehat W_K)
=
e^\top M_K^0e+\sum_x\mathcal R_x.
\label{eq:rope-total-kl-identity}
\end{equation}

The pre-fitting regularization contributes
\[
\epsilon e^\top P_Ke
=
\epsilon\sum_x
e^\top\bigl((q_xq_x^\top)\otimes I_d\bigr)e
=
\epsilon\sum_x\|Eq_x\|_2^2.
\]
Hence
\begin{equation}
\sum_x\mathcal L_{K,x}(\widehat W_K)
=
e^\top M_K^\epsilon e
+
\sum_x\mathcal R_x
-
\Delta_{\epsilon,\mathrm{pre}}(E).
\label{eq:rope-prefitting-kl-identity}
\end{equation}
The subtraction is essential: the pre-fitting ridge is
added to the fitting target, not to the original KL loss.

\paragraph{Step 4: Kronecker approximation and post-fitting ridge.}
The vectorization identity gives
\[
e^\top(A_K\otimes B_K)e
=
\operatorname{tr}(B_KEA_KE^\top).
\]
Using the Cholesky factorizations,
\begin{align*}
\widehat{\mathcal L}_K(\widehat W_K)
&=
\operatorname{tr}
\left(
(B_K+\epsilon I_d)
E
(A_K+\epsilon I_{d_h})
E^\top
\right)
\\
&=
e^\top(A_K\otimes B_K)e
+
\Delta_{\epsilon,\mathrm{post}}(E).
\end{align*}
The pre-fitting correction is nonnegative because
$P_K\succeq0$, and the post-fitting correction is
nonnegative because $A_K,B_K\succeq0$.

Combining this equality with
\eqref{eq:rope-prefitting-kl-identity} yields the exact
signed error decomposition
\begin{equation}
\begin{aligned}
\sum_x\mathcal L_{K,x}(\widehat W_K)
-
\widehat{\mathcal L}_K(\widehat W_K)
&=
e^\top(M_K^\epsilon-A_K\otimes B_K)e
\\
&\quad+
\sum_x\mathcal R_x
-
\Delta_\epsilon(E).
\end{aligned}
\label{eq:rope-exact-error-decomposition}
\end{equation}
Taking absolute values and applying
\eqref{eq:rope-explicit-taylor-bound} proves
\eqref{eq:rope-key-error-bound}.

When RoPE is disabled,
\[
B_x=q_x^\top\otimes X_x,
\qquad
B_x^\top J_xB_x
=
(q_xq_x^\top)\otimes(X_x^\top J_xX_x).
\]
Adding $\epsilon P_K$ therefore gives
\eqref{eq:rope-prefitting-nope-reduction}.

\paragraph{Step 5: Closed-form surrogate minimizer.}
Since $A_K+\epsilon I_{d_h}$ and $B_K+\epsilon I_d$
are positive definite, $L_K$ and $C_K$ are invertible.
Define
\[
T_K=L_KW_KC_K,
\qquad
Y=L_K\widehat W_KC_K.
\]
The map $\widehat W_K\mapsto Y$ is bijective and
preserves rank. Consequently,
\[
\min_{\operatorname{rank}(\widehat W_K)\le r}
\widehat{\mathcal L}_K(\widehat W_K)
=
\min_{\operatorname{rank}(Y)\le r}
\|T_K-Y\|_F^2.
\]
By the Eckart--Young--Mirsky theorem, a global minimizer is
$Y^\star=U_r\Sigma_rV_r^\top$.
Applying the inverse transformation proves
\eqref{eq:rope-key-svd-solution}.
\end{proof}

\section{Mathematical Comparison with A$^3$}
\label{app:comparison-a3}

A$^3$~\citep{wong2025a3} and BELT both use analytical
low-rank approximation, but differ in their treatment of
key compression.
A$^3$ approximates the attention logits before softmax,
whereas BELT derives its key metric from preservation of
the attention distributions after softmax.
The distinction lies in the approximation objective and
its induced metric, rather than in the truncated-SVD
operation itself.
Our value-side construction is closely related to the
attention-aware weighted approximation in A$^3$; we
therefore focus this comparison on keys.

\subsection{Notation and the Pre-Softmax Attention Operator}

We use the projection convention of
Section~\ref{sec:prelim:attn}:
$W_Q,W_K\in\mathbb R^{d\times d_h}$ and token
representations are row vectors.
Without RoPE, the logits for query $x$ and context $X$ are
\[
z_x(W_K)
=
\frac{1}{\sqrt{d_h}}xW_QW_K^\top X^\top.
\]
Thus the bilinear attention operator in representation
space is
\[
P=W_QW_K^\top\in\mathbb R^{d\times d}.
\]
Under the alternative convention in which projection
matrices act on column-vector representations,
the same operator is written as $W_Q^\top W_K$.
The apparent transpose difference is therefore purely
notational.

To describe covariance-weighted logit reconstruction,
let $Y$ stack calibration query representations and
let $X$ stack context representations.
The corresponding logit matrix is
\[
Z=\frac{1}{\sqrt{d_h}}YPX^\top.
\]

\subsection{A$^3$: Covariance-Weighted Logit Approximation}
\label{app:a3-logit-objective}

The pre-softmax reconstruction criterion underlying
the A$^3$ key construction can be written as
\begin{equation}
\mathcal L_{\mathrm{logit}}(\widehat P)
=
\frac{1}{d_h}
\|Y(P-\widehat P)X^\top\|_F^2.
\label{eq:a3-logit-objective}
\end{equation}
Let
\[
G_Q=Y^\top Y,
\qquad
G_X=X^\top X,
\]
and write covariance factorizations
\[
G_Q=U_QU_Q^\top,
\qquad
G_X=U_XU_X^\top.
\]
Then trace cyclicity gives
\begin{align}
\mathcal L_{\mathrm{logit}}(\widehat P)
&=
\frac{1}{d_h}
\operatorname{tr}
\left(
G_Q(P-\widehat P)
G_X(P-\widehat P)^\top
\right)
\nonumber\\
&=
\frac{1}{d_h}
\|U_Q^\top(P-\widehat P)U_X\|_F^2.
\label{eq:a3-whitened-logit-objective}
\end{align}

When the covariance factors are invertible, minimizing
\eqref{eq:a3-whitened-logit-objective} over
$\operatorname{rank}(\widehat P)\le r$ reduces to
truncated SVD of
\[
U_Q^\top P U_X
=
U_Q^\top W_QW_K^\top U_X.
\]
In particular,
\[
\widehat P_r
=
U_Q^{-\top}
\left(U_Q^\top P U_X\right)_r
U_X^{-1}.
\]
When query and context statistics coincide,
$U_Q=U_X$, giving the same-covariance transformed
operator.
With column-oriented projection matrices, this becomes
the expression $U_X^\top W_Q^\top W_KU_X$.

The covariance weighting accounts for the calibration
distribution of token representations.
However, the reconstruction criterion remains Euclidean
in logit space: for a fixed query, its error contribution
is the sum of squared perturbations of the logits.
It does not explicitly distinguish logit perturbations
according to their effect on the resulting softmax
distribution.

The formula above describes the unregularized
covariance-weighted objective.
If positive-definite covariance regularization is used,
the SVD instead exactly solves the corresponding
regularized weighted objective.

\subsection{Expressing Logit Reconstruction in Key-Weight Space}

For a direct comparison with BELT, hold $W_Q$ fixed and
write
\[
E=W_K-\widehat W_K,
\qquad e=\operatorname{vec}(E).
\]
The logit-perturbation column vector for query $x$ is
\[
h_x
=
z_x(\widehat W_K)^\top-z_x(W_K)^\top
=
-\frac{1}{\sqrt{d_h}}XEq_x^\top.
\]
Here $q_x=xW_Q$ retains the row-vector convention of
the preliminaries, and $X$ denotes the context of the
current query.

The resulting logit-reconstruction objective is
\begin{equation}
\mathcal L_{\mathrm{logit}}(\widehat W_K)
=
\sum_x\|h_x\|_2^2
=
e^\top M_{\mathrm{logit}}e,
\label{eq:a3-key-space-logit-objective}
\end{equation}
where
\begin{equation}
M_{\mathrm{logit}}
=
\frac{1}{d_h}
\sum_x
(q_x^\top q_x)\otimes(X^\top X).
\label{eq:a3-key-space-logit-metric}
\end{equation}
If all queries share the same context matrix, this
metric factorizes exactly:
\[
M_{\mathrm{logit}}
=
\frac{1}{d_h}
\left(\sum_x q_x^\top q_x\right)
\otimes(X^\top X).
\]
This separability explains why covariance whitening
naturally leads to a weighted-SVD solution for logit
reconstruction.

For query-dependent causal supports, $X$ varies with
$x$, and the metric is generally a sum of Kronecker
products rather than one product.
Thus a single covariance-factor construction corresponds
to the particular calibration weighting and support
treatment used in that construction.

The bilinear-operator formulation
\eqref{eq:a3-logit-objective} and the fixed-query
key-weight formulation
\eqref{eq:a3-key-space-logit-objective} should not be
identified without specifying the feasible
parameterization.
The former optimizes a rank-constrained operator
$\widehat P$, whereas the latter restricts that operator
to $W_Q\widehat W_K^\top$ with fixed $W_Q$.
Our comparison of error metrics does not require the two
parameterizations to be identical.

\subsection{BELT: Attention-Distribution Curvature}
\label{app:belt-attention-curvature}

BELT instead starts from
\begin{equation}
\mathcal L_K(\widehat W_K)
=
\sum_x
D_{\mathrm{KL}}
\left(
a_x(W_K)\,\middle\|\,a_x(\widehat W_K)
\right).
\label{eq:belt-comparison-kl-objective}
\end{equation}
For the original attention row vector $a_x$, let
\[
J_x=\operatorname{diag}(a_x)-a_x^\top a_x.
\]
The direct logit-space expansion gives
\begin{equation}
\mathcal L_K(\widehat W_K)
=
\frac12\sum_x h_x^\top J_xh_x
+
\sum_x\mathcal R_x(E),
\label{eq:belt-comparison-kl-expansion}
\end{equation}
with the explicit remainder bound derived in
Lemma~\ref{lemma:lk_frob_per_token}.

The unregularized quadratic metric is therefore
\begin{equation}
M_K^0
=
\frac{1}{2d_h}
\sum_x
(q_x^\top q_x)\otimes(X^\top J_xX).
\label{eq:belt-comparison-key-metric}
\end{equation}
Comparing \eqref{eq:a3-key-space-logit-metric} and
\eqref{eq:belt-comparison-key-metric}, the essential
difference is
\[
\boxed{
X^\top X
\quad\longrightarrow\quad
X^\top J_xX.
}
\]
Apart from the overall factor $1/2$, logit reconstruction
uses the identity metric on logit perturbations, whereas
BELT uses the query-dependent softmax curvature $J_x$.

The weighting can also be expressed as
\begin{equation}
X^\top J_xX
=
\sum_n a_{x,n}X_n^\top X_n
-
(a_xX)^\top(a_xX).
\label{eq:belt-attention-context-covariance}
\end{equation}
Thus BELT uses an attention-weighted, centered covariance
of context representations.
This differs from the uncentered context second moment
$X^\top X$ used by Euclidean logit reconstruction.

\subsection{What the Different Metrics Preserve}
\label{app:a3-belt-metric-interpretation}

\paragraph{Uniform logit shifts.}
Softmax is invariant to adding the same scalar to every
logit:
\[
\operatorname{softmax}(z+c\boldsymbol 1)
=
\operatorname{softmax}(z).
\]
For $h_x=c\boldsymbol 1$, logit reconstruction incurs
\[
\|h_x\|_2^2=n_xc^2,
\]
whereas
\[
h_x^\top J_xh_x=0.
\]
In this case the original attention KL is also exactly
zero, not merely zero to second order.
BELT's unregularized curvature therefore respects a
symmetry of the attention distribution that raw logit
reconstruction does not.

\paragraph{Attention-weighted variation.}
For an arbitrary perturbation,
\begin{equation}
h_x^\top J_xh_x
=
\sum_n a_{x,n}h_{x,n}^2
-
\left(\sum_n a_{x,n}h_{x,n}\right)^2.
\label{eq:belt-variance-interpretation}
\end{equation}
Equivalently,
\begin{equation}
h_x^\top J_xh_x
=
\frac12\sum_{n,m}
a_{x,n}a_{x,m}(h_{x,n}-h_{x,m})^2.
\label{eq:belt-pairwise-interpretation}
\end{equation}
The metric therefore measures changes in relative logits,
weighted by the original attention probabilities,
rather than absolute logit discrepancies.

For example, with two attended positions and
$a_x=(p,1-p)$,
\[
h_x^\top J_xh_x
=
p(1-p)(h_{x,1}-h_{x,2})^2,
\]
while Euclidean logit reconstruction penalizes
$h_{x,1}^2+h_{x,2}^2$.
The two criteria thus need not select the same low-rank
subspace.

\paragraph{Uniform-attention special case.}
If $a_x=\boldsymbol 1^\top/n_x$, then
\[
J_x
=
\frac{1}{n_x}
\left(
I-\frac{1}{n_x}\boldsymbol 1\boldsymbol 1^\top
\right).
\]
Hence BELT's local objective becomes
\[
\frac12h_x^\top J_xh_x
=
\frac{1}{2n_x}
\left\|
h_x-\frac{\boldsymbol 1^\top h_x}{n_x}
\boldsymbol 1
\right\|_2^2.
\]
In this special case, it is proportional to centered
logit reconstruction.
For nonuniform attention, however, the weighting is
generally anisotropic and query dependent.

\paragraph{Local sensitivity and its limitation.}
The attention-dependent metric can assign little
quadratic cost to directions that are locally
insensitive under the original attention distribution.
This does not imply that arbitrarily large perturbations
in those directions are harmless.
For example, a sufficiently large change to a
low-probability logit can substantially alter attention.
The Taylor remainder is therefore essential when
assessing the accuracy of the quadratic approximation.

Pre-fitting regularization and subsequent Kronecker
approximation can also modify the null directions and
relative weights of the exact curvature.
The interpretations above describe $M_K^0$; the
implemented surrogate approximates this metric with the
explicit error contributions analyzed in our theorems.

\subsection{RoPE and Position-Dependent Key Sensitivity}
\label{app:a3-belt-rope}

Without RoPE, a single bilinear operator
$P=W_QW_K^\top$ maps query and context representations
to logits.
With RoPE, the corresponding operator depends on the
relative position:
\[
z_{x,n}(W_K)
=
\frac{1}{\sqrt{d_h}}
xW_QR_{n-m_x}W_K^\top X_n^\top.
\]
In general, there is no single position-independent
matrix $P$ representing all these interactions.

Let $B_x$ be the RoPE-aware perturbation map defined in
Lemma~\ref{lemma:lk_kron_rank1_rope}, so that
\[
h_x=-\frac{1}{\sqrt{d_h}}B_xe.
\]
A RoPE-aware Euclidean logit objective would have metric
\begin{equation}
M_{\mathrm{logit}}^{\mathrm{RoPE}}
=
\frac{1}{d_h}\sum_x B_x^\top B_x,
\label{eq:rope-logit-comparison-metric}
\end{equation}
whereas BELT's exact local KL metric is
\begin{equation}
M_K^{0,\mathrm{RoPE}}
=
\frac{1}{2d_h}\sum_x B_x^\top J_xB_x.
\label{eq:rope-kl-comparison-metric}
\end{equation}
Thus the distinction between identity weighting and
softmax-curvature weighting persists even when both
criteria account for positional rotations.

Equation~\eqref{eq:rope-logit-comparison-metric} describes
a RoPE-aware logit-reconstruction criterion for comparison;
it is not an assertion that A$^3$ implements this extension.
BELT explicitly incorporates the relative rotations into
the perturbation map used to construct its key metric.

\subsection{Tractability and the Scope of Optimality}
\label{app:a3-belt-optimality}

A$^3$ obtains an analytical approximation by applying SVD
to a covariance-weighted pre-softmax operator.
BELT's query-dependent attention curvature is generally
not separable into a single left and right weighting
matrix.
To obtain a tractable weighted-SVD problem, BELT first
adds the specified pre-fitting ridge,
\[
M_K^\epsilon
=
M_K^0
+
\epsilon\sum_x(q_x^\top q_x)\otimes I_d,
\]
and fits
\[
M_K^\epsilon\approx A_K\otimes B_K.
\]
After regularizing the fitted factors, the resulting
surrogate is
\[
\mathcal L_K^\approx(\widehat W_K)
=
\|L_K(W_K-\widehat W_K)C_K\|_F^2.
\]
Its global rank-$r$ minimizer is
\[
\widehat W_{K,r}
=
L_K^{-1}(L_KW_KC_K)_rC_K^{-1}.
\]

Both methods therefore use standard weighted low-rank
approximation once their respective metrics are fixed.
The distinction is that A$^3$ targets pre-softmax
reconstruction, while BELT constructs its metric from
post-softmax distribution preservation and then
approximates that metric to recover a tractable SVD
solution.
The Kronecker fitting stage is separate from the
closed-form SVD stage.

BELT does not establish that its solution globally
minimizes the original attention KL, nor that it
universally dominates logit reconstruction.
Its approximation involves Taylor, Kronecker, and
regularization errors, which must be assessed at the
returned solutions.
Moreover, attention KL does not directly incorporate
the value vectors or downstream network sensitivity;
its connection to attention-output preservation is
through the conditional bound in
Appendix~\ref{objective_derivation}.

\section{Additional method details}

\subsection{Budget-Constrained Gates and Their Gradients}
\label{app:budget-gates}

Let
\[
\mathcal F=\{(i,j):1\le i\le N,\ 2\le j\le m_i\}
\]
be the set of components eligible for allocation after
reserving the first component of every block.
Write
\[
M_{\mathrm{free}}=|\mathcal F|,
\qquad b=B-N,
\qquad t_{i,j}=s_i-j.
\]
For an interior budget $0<b<M_{\mathrm{free}}$, define
$\alpha(\boldsymbol s)$ through
\begin{equation}
F(\alpha,\boldsymbol s)
=
\sum_{(i,j)\in\mathcal F}
\sigma\left(\frac{s_i-j-\alpha}{\tau}\right)
-b
=
0.
\label{eq:allocation-cutoff}
\end{equation}

\paragraph{Existence and uniqueness.}
For finite scores and $\tau>0$, $F$ is continuous and strictly
decreasing in $\alpha$.
Furthermore,
\[
\lim_{\alpha\to-\infty}F(\alpha,\boldsymbol s)
=
M_{\mathrm{free}}-b>0,
\qquad
\lim_{\alpha\to+\infty}F(\alpha,\boldsymbol s)
=
-b<0.
\]
Therefore, \eqref{eq:allocation-cutoff} has a unique finite
solution.

\paragraph{Implicit differentiation.}
For $(i,j)\in\mathcal F$, let
\[
g_{i,j}
=
\sigma\left(\frac{s_i-j-\alpha}{\tau}\right),
\qquad
w_{i,j}
=
\frac{g_{i,j}(1-g_{i,j})}{\tau},
\]
and define
\[
W_i=\sum_{j=2}^{m_i}w_{i,j},
\qquad
W=\sum_i W_i>0.
\]
Holding $\alpha$ fixed,
\[
\frac{\partial F}{\partial s_k}=W_k,
\qquad
\frac{\partial F}{\partial\alpha}=-W.
\]
The implicit function theorem gives
\begin{equation}
\frac{\partial\alpha}{\partial s_k}
=
-\frac{\partial F/\partial s_k}
        {\partial F/\partial\alpha}
=
\frac{W_k}{W}.
\label{eq:cutoff-gradient}
\end{equation}
Consequently,
\begin{equation}
\frac{\partial g_{i,j}}{\partial s_k}
=
w_{i,j}
\left(
\mathbf 1_{\{i=k\}}-\frac{W_k}{W}
\right).
\label{eq:gate-gradient}
\end{equation}
Summing over the free components yields
\[
\sum_{(i,j)\in\mathcal F}
\frac{\partial g_{i,j}}{\partial s_k}
=
W_k-\frac{W_k}{W}W=0,
\]
consistent with the constant relaxed budget.

If $\ell_{i,j}=\partial\mathcal L_{\mathrm{KD}}/
\partial g_{i,j}$ and the scores affect the loss only through
the gates, the score gradient is
\begin{equation}
\frac{\partial\mathcal L_{\mathrm{KD}}}{\partial s_k}
=
\sum_{j=2}^{m_k}\ell_{k,j}w_{k,j}
-
\frac{W_k}{W}
\sum_{(i,j)\in\mathcal F}\ell_{i,j}w_{i,j}.
\label{eq:allocation-score-gradient}
\end{equation}
The reserved gates $g_{i,1}=1$ have zero derivatives.

\paragraph{Numerical solution.}
The forward cutoff can be computed using bisection or
safeguarded Newton iteration.
The backward pass uses \eqref{eq:cutoff-gradient}--%
\eqref{eq:allocation-score-gradient}, rather than treating
the cutoff as constant.
In finite precision, the relaxed budget is satisfied up to
the stated root-solving tolerance.
Very small temperatures can cause sigmoid saturation and
an ill-conditioned cutoff derivative; the implementation
should monitor the residual and numerical stability.

\paragraph{Endpoint budgets.}
When $b=0$, all free gates are set to zero.
When $b=M_{\mathrm{free}}$, all free gates are set to one.
These cases require no cutoff solve or allocation training.

\paragraph{Hard allocation and prefix preservation.}
For deployment, retain the first component of every block
and select exactly $b$ elements of $\mathcal F$ with the
largest scores, breaking ties deterministically.
For two free components in the same block with $j<k$,
\[
t_{i,j}=s_i-j>s_i-k=t_{i,k}.
\]
Thus selecting component $k$ necessarily selects all
preceding free components.
Together with the reserved first component, each block
retains a prefix, and the final allocation has exactly $B$
components.

The relaxed gate-mass constraint does not imply a hard
rank constraint during training.
Nor does exact budget satisfaction establish optimality of
the learned hard allocation.
Its quality and the soft-to-hard discrepancy are evaluated
empirically.

\subsection{Integration with Quantization}
\label{app:quant}

To further enhance compression beyond low-rank decomposition, we quantize the compressed KV cache. \ours{} determines only the factors and the rank of each block, and block \(i\) still caches the compressed representation \(Z_i=X\Wdown_i\in\mathbb{R}^{n\times r_i}\) as in prior low-rank methods. Hence existing quantizers for low-rank KV caches apply to \ours{} without modification. We combine \ours{} with the quantizers of Palu~\citep{chang2024palu} and STAR-KV~\citep{bhatnagar2026starkv}, denoted Palu-Q and STAR-Q.

Since the \(j\)-th column of \(\Wdown_i\) is scaled by the square root of the \(j\)-th largest singular value (Theorem~\ref{thm:value_bsvd} and Corollary~\ref{thm:key_bsvd}), the leading coordinates of \(Z_i\) are much larger than the others and dominate the per-token quantization range. Both quantizers therefore apply an orthogonal rotation that spreads these outliers across coordinates~\citep{ashkboos2024quarot}, followed by per-token asymmetric quantization. With \(\mathcal Q_\beta\) denoting \(\beta\)-bit quantization, Palu-Q and STAR-Q cache
\begin{small}
    \begin{align}\label{eq: latent_quant}
    \hat Z_i = \mathcal Q_\beta(Z_iR_i)
    \quad\text{and}\quad
    \hat Z_i = \bigl[\,\mathcal Q_{\beta_{\mathrm{out}}}(Z_i^{\mathrm{out}}R_i^{\mathrm{out}}) \;\; \mathcal Q_{\beta_{\mathrm{in}}}(Z_i^{\mathrm{in}}R_i^{\mathrm{in}})\,\bigr],
    \end{align}
\end{small}
respectively, where \(R_i\), \(R_i^{\mathrm{out}}\), and \(R_i^{\mathrm{in}}\) are orthogonal matrices. STAR-Q quantizes the leading \(20\%\) coordinates \(Z_i^{\mathrm{out}}\) with \(\beta_{\mathrm{out}}=4\) bits and the rest \(Z_i^{\mathrm{in}}\) with \(\beta_{\mathrm{in}}=3\) bits, which requires the leading coordinates to be the dominant ones. \ours{} satisfies this requirement because every hard allocation retains a prefix of the transformed SVD of each block (Section~\ref{sec:method:alloc}). The rotations are fused into the factors offline, i.e., \(\Wdown_i\leftarrow\Wdown_iR_i\) and \(\Wup_i\leftarrow R_i^\top\Wup_i\) with a block-diagonal \(R_i\) for STAR-Q, and thus add no computation at inference. As the learned ranks are arbitrary integers, we use the normalized Walsh--Hadamard matrix when the rotated width is a power of two and a fixed random orthogonal matrix otherwise. As a rank can be as small as one, some blocks are single-signed, so we do not clip the zero point of \(\mathcal Q_\beta\) to \([0,2^\beta-1]\). 

\section{Additional experimental settings}
\paragraph{Configurations for System-level Performance Analysis}\label{app:system_configure}
All low-rank methods share STAR-KV's Triton kernels~\citep{bhatnagar2026starkv}, which we generalize to per-head, grouped, and joint layouts: keys are rebuilt with RoPE on chip, attention weights are applied directly to the compressed values, and packed codes are dequantized on chip with STAR-Q. Since latency depends only on tensor shapes, each cache holds random values in the layout and ranks of the corresponding checkpoint.

\section{Additional experimental results}\label{app: add_res}

\subsection{Additional system-level analysis results}\label{app:system}
\begin{figure}[h]
\centering
\includegraphics[width=\textwidth]{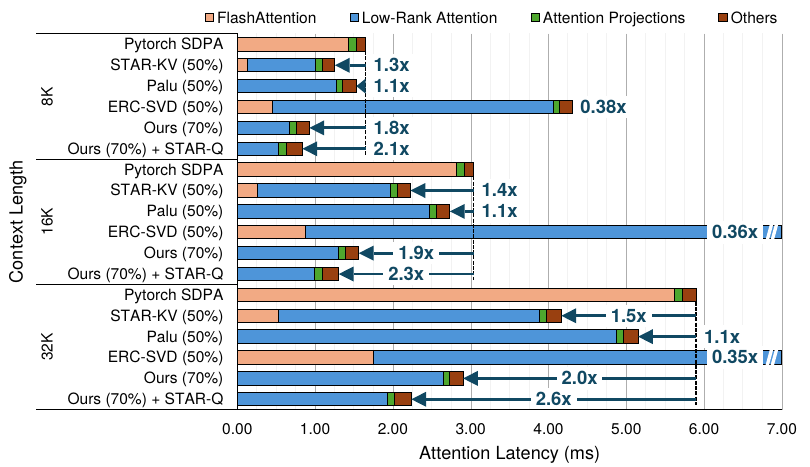}
\caption{Latency breakdown of the attention module in Figure~\ref{fig:attn_speedup}, averaged over all 32 layers. Arrows give the speedup over SDPA; the ERC-SVD bars at 16K and 32K are truncated.}
\label{fig:latency_breakdown}
\end{figure}

\subsection{Additional ablation studies}\label{app: ablation}

\paragraph{Attention without RoPE} To test \ours{} without rotary embeddings, we follow A3~\citep{wong2025a3} and use MPT-7B, which encodes position via ALiBi~\citep{press2022alibi}. We keep the main pipeline and data, setting all sequence lengths to 2048. Table~\ref{tab:nope} shows that \ours{} incurs accuracy losses close to those on LLaMA-2-7B, the other multi-head attention model, at both compression rates. \textbf{These results confirm that \ours{} generalizes to attention without RoPE.}

\begin{table*}[h]
\centering
\caption{\ours{} on MPT-7B, which encodes position with ALiBi instead of RoPE. Perplexity is measured on 2048-token windows.}
\label{tab:nope}
\resizebox{\textwidth}{!}{
\begin{tabular}{lcccccccccc}
\toprule
\multirow{2}{*}{\textbf{Model}} & \multirow{2}{*}{\textbf{Comp (\%)}} & \multicolumn{2}{c}{\textbf{Perplexity} $\downarrow$} & \multicolumn{7}{c}{\textbf{Zero-Shot Accuracy (\%)} $\uparrow$} \\
\cmidrule(lr){3-4} \cmidrule(lr){5-11}
 & & Wiki2 & C4 & OBQA & PIQA & ARC-e & ARC-c & Hella & Wino & Avg. \\
\midrule
MPT-7B & 0 & 7.70 & 9.99 & 42.60 & 79.22 & 74.66 & 41.81 & 76.15 & 69.14 & 63.93 \\
\cmidrule(lr){1-11}
\multirow{2}{*}{Ours} & 50 & 7.83 & 10.76 & 41.20 & 78.78 & 73.23 & 40.36 & 74.83 & 67.72 & 62.69 \\
 & 70 & 8.34 & 12.06 & 38.80 & 76.44 & 68.98 & 39.08 & 71.13 & 66.06 & 60.08 \\
\bottomrule
\end{tabular}
}
\end{table*}

\paragraph{Effect of Rank Allocation.} To isolate rank allocation, we fix the key and value bases of \ours{} on LLaMA-3.1-8B-Instruct and vary only the rank distribution under the same KV cache budget, which is reported in Table \ref{tab:rank_allocation}. Baselines include Static (uniform fractions), Layer-wise search (ASVD/A3 sensitivity-based), and STAR-KV (learned soft thresholds). \textbf{\ours{} achieves the lowest perplexity and highest average accuracy at both compression rates.} Layer-wise search improves over Static but still trails \ours{}, while the state-of-the-art STAR-KV falls below Static. We give an analysis of why START-KV fails in Appendix \ref{appendix:alpha}. A key pattern is that \textbf{learned allocation shifts ranks from keys to values}, retaining only 39--41\% in keys versus about half for layer-wise search, and \textbf{its benefit grows with compression}: at 70\%, it raises average accuracy by over 10 points over Static with identical bases.

\begin{table*}[h]
\centering
\caption{Compressing only the keys or only the values on LLaMA-3.1-8B-Instruct. Each variant takes the \ours{} checkpoint at 50\% or 70\% compression, keeps its low-rank projections for the compressed cache and the original projections for the other cache, and is not retrained. Comp is the resulting KV cache compression.}
\label{tab:kv_compression}
\resizebox{\textwidth}{!}{
\begin{tabular}{llcccccccccc}
\toprule
\multirow{2}{*}{\textbf{Method}} & \multirow{2}{*}{\textbf{Compressed cache}} & \multirow{2}{*}{\textbf{Comp (\%)}} & \multicolumn{2}{c}{\textbf{Perplexity} $\downarrow$} & \multicolumn{7}{c}{\textbf{Zero-Shot Accuracy (\%)} $\uparrow$} \\
\cmidrule(lr){4-5} \cmidrule(lr){6-12}
 & & & Wiki2 & C4 & OBQA & PIQA & ARC-e & ARC-c & Hella & Wino & Avg. \\
\midrule
LLaMA-3.1-8B-Inst & -- & 0 & 6.75 & 11.10 & 43.00 & 79.92 & 81.82 & 55.12 & 79.20 & 73.56 & 68.77 \\
\cmidrule(lr){1-12}
\multirow{3}{*}{Ours (50\%)} & Keys and values & 50.0 & 7.10 & 12.99 & 44.20 & 79.60 & 79.63 & 53.92 & 77.67 & 71.82 & 67.81 \\
 & Keys only & 30.3 & 6.90 & 11.99 & 44.20 & 80.30 & 81.40 & 54.78 & 78.37 & 73.24 & 68.72 \\
 & Values only & 19.7 & 6.97 & 12.17 & 44.80 & 79.22 & 80.89 & 54.52 & 78.10 & 73.56 & 68.52 \\
\cmidrule(lr){1-12}
\multirow{3}{*}{Ours (70\%)} & Keys and values & 70.0 & 8.54 & 17.51 & 41.80 & 76.66 & 73.32 & 47.01 & 72.70 & 68.75 & 63.37 \\
 & Keys only & 37.6 & 7.46 & 13.39 & 42.20 & 80.14 & 79.38 & 52.05 & 76.76 & 70.48 & 66.83 \\
 & Values only & 32.4 & 8.12 & 16.44 & 43.20 & 76.71 & 77.31 & 50.60 & 73.91 & 71.51 & 65.54 \\
\bottomrule
\end{tabular}
}
\end{table*}

\paragraph{Effect of Key and Value Compression.} To isolate each cache, we compress only the keys or only the values from our checkpoints at 50\% and 70\% compression, without retraining. The results of this experiment are included in Table \ref{tab:kv_compression}. First, we observe that \textbf{Values are substantially more sensitive to compression than keys.} At 70\%, compressing only values degrades accuracy and perplexity more than compressing only keys, even though the latter removes more of the cache. This matches our learned allocation, which assigns fewer ranks to keys (39--41\%). At high compression, the combined loss of both caches roughly equals the sum of individual losses, confirming that most degradation comes from the value cache and that keys tolerate much lower ranks.

\begin{table*}[t]
\centering
\caption{Rank allocation on LLaMA-3.1-8B-Instruct with the key and value bases of \ours{} fixed. All methods share the same KV cache budget; STAR-KV stops once its budgets are met and uses 1.5\% and 0.3\% fewer ranks at 50\% and 70\% compression. Bold marks the best result at each compression rate.}
\label{tab:rank_allocation}
\resizebox{\textwidth}{!}{
\begin{tabular}{lcccccccccc}
\toprule
\multirow{2}{*}{\textbf{Method}} & \multirow{2}{*}{\textbf{Comp (\%)}} & \multicolumn{2}{c}{\textbf{Perplexity} $\downarrow$} & \multicolumn{7}{c}{\textbf{Zero-Shot Accuracy (\%)} $\uparrow$} \\
\cmidrule(lr){3-4} \cmidrule(lr){5-11}
 & & Wiki2 & C4 & OBQA & PIQA & ARC-e & ARC-c & Hella & Wino & Avg. \\
\midrule
LLaMA-3.1-8B-Inst & 0 & 6.75 & 11.10 & 43.00 & 79.92 & 81.82 & 55.12 & 79.20 & 73.56 & 68.77 \\
\cmidrule(lr){1-11}
\multirow{2}{*}{Static} & 50 & 7.66 & 14.92 & 42.20 & 78.35 & 76.73 & 50.77 & 74.44 & 72.53 & 65.84 \\
 & 70 & 14.33 & 38.64 & 34.80 & 71.06 & 62.12 & 34.56 & 55.05 & 61.64 & 53.20 \\
\cmidrule(lr){1-11}
\multirow{2}{*}{Layer-wise search} & 50 & 7.52 & 14.65 & 42.40 & 78.51 & 78.58 & 53.33 & 75.62 & \textbf{73.01} & 66.91 \\
 & 70 & 11.57 & 32.02 & 37.20 & 72.80 & 68.77 & 41.04 & 61.84 & 65.19 & 57.81 \\
\cmidrule(lr){1-11}
\multirow{2}{*}{STAR-KV} & 50 & 8.18 & 16.21 & 40.60 & 77.31 & 75.51 & 46.59 & 71.07 & 68.90 & 63.33 \\
 & 70 & 15.42 & 40.21 & 35.00 & 71.33 & 60.61 & 31.40 & 54.45 & 58.64 & 51.90 \\
\cmidrule(lr){1-11}
\multirow{2}{*}{Ours} & 50 & \textbf{7.10} & \textbf{12.99} & \textbf{44.20} & \textbf{79.60} & \textbf{79.63} & \textbf{53.92} & \textbf{77.67} & 71.82 & \textbf{67.81} \\
 & 70 & \textbf{8.54} & \textbf{17.51} & \textbf{41.80} & \textbf{76.66} & \textbf{73.32} & \textbf{47.01} & \textbf{72.70} & \textbf{68.75} & \textbf{63.37} \\
\bottomrule
\end{tabular}
}
\end{table*}

\paragraph{Effect of Calibration Data.} We evaluate the sensitivity of \ours\ on the calibration data, which is performed on WikiText-2 and reported in Table \ref{tab:dataset}. For this experiment, we rebuild \ours{} for Qwen3-32B at 50\% compression using FineWeb-Edu~\citep{penedo2024fineweb} and the C4~\citep{c4} training split instead, keeping all else fixed. Table \ref{tab:dataset} suggests that \textbf{accuracy is nearly invariant, varying by only 0.26 points across corpora and consistently exceeding the uncompressed model.} Perplexity follows the calibration domain: WikiText-2 yields the lowest WikiText-2 perplexity, while web corpora lower C4 perplexity. The learned rank allocation also preserves its structure. \textbf{These results show that \ours{} is robust to the choice of calibration data.}

\begin{table*}[h]
\centering
\caption{Sensitivity of \ours{} to the calibration corpus on Qwen3-32B at 50\% KV cache compression. Each corpus supplies both the calibration set of the key and value bases and the rank-allocation training data. WikiText-2 is the corpus of the main experiments (Table~\ref{tab:main_results}).}
\label{tab:dataset}
\resizebox{\textwidth}{!}{
\begin{tabular}{llccccccccc}
\toprule
\multirow{2}{*}{\textbf{Model}} & \multirow{2}{*}{\textbf{Calibration data}} & \multicolumn{2}{c}{\textbf{Perplexity} $\downarrow$} & \multicolumn{7}{c}{\textbf{Zero-Shot Accuracy (\%)} $\uparrow$} \\
\cmidrule(lr){3-4} \cmidrule(lr){5-11}
 & & Wiki2 & C4 & OBQA & PIQA & ARC-e & ARC-c & Hella & Wino & Avg. \\
\midrule
Qwen3-32B & -- & 7.02 & 11.98 & 45.80 & 80.90 & 84.43 & 60.84 & 82.61 & 73.16 & 71.29 \\
\cmidrule(lr){1-11}
\multirow{3}{*}{Ours (50\%)} & WikiText-2 & 7.01 & 12.35 & 45.40 & 80.25 & 84.18 & 64.08 & 82.50 & 75.06 & 71.91 \\
 & FineWeb-Edu & 7.11 & 12.10 & 46.00 & 80.58 & 85.10 & 63.57 & 82.23 & 74.19 & 71.94 \\
 & C4 & 7.10 & 12.07 & 47.20 & 80.09 & 84.72 & 62.29 & 82.06 & 73.72 & 71.68 \\
\bottomrule
\end{tabular}
}
\end{table*}

\subsection{Empirical verification for theoretical claims}\label{app? emp_verify}

\paragraph{Evaluation of the bounds for approximating key-side objectives}

We empirically evaluate the error terms in the bound used for bounding the approximation error in Theorem \ref{thm:lk_kron_rank1} when $\widehat{W_K}=\widehat{W_K}^*$. The evaluation results are included in Table \ref{tab:key_approx}, which suggests that these error terms are all quite small, and the sum of Kronecker and Taylor approximation error is also very close to the true gap between $L_K^{\approx}$ and $L_K$, thus empirically justifying the theoretical analysis in Theorem \ref{thm:lk_kron_rank1}.

\begin{table}[]
    \centering
    \begin{tabular}{c|c}\toprule
        Kronecker truncation error & 0.015933 \\ 
        Taylor approximation error & -0.004157\\
        $\epsilon$ & 6.6e-5 \\
        $|\mathcal L_K(\widehat W_K 
-
\mathcal L_K^{\approx}(\widehat W_K)|$ & 0.011775\\ \bottomrule
    \end{tabular}
    \caption{Evaluation of the bounds used in Theorem \ref{thm:lk_kron_rank1} at $\widehat{W_K}^*$}
    \label{tab:key_approx}
\end{table}



\subsection{Per-task main results}
\label{per-task-main}

\begin{table*}[t]
\centering
\small
\caption{Zero-shot evaluation on commonsense reasoning benchmarks using LM-Eval-Harness. We report accuracy (\%) under different KV cache compression rates, together with perplexity on WikiText-2 and C4. A3 does not apply to Qwen3-32B. Bold marks the best result at each compression rate.}
\label{tab:main_results}\label{tab:main_results_others}
\resizebox{\textwidth}{!}{
\begin{tabular}{lcccccccccc}
\toprule
\multirow{2}{*}{\textbf{Model}} & \multirow{2}{*}{\textbf{Comp (\%)}} & \multicolumn{2}{c}{\textbf{Perplexity} $\downarrow$} & \multicolumn{7}{c}{\textbf{Zero-Shot Accuracy (\%)} $\uparrow$} \\
\cmidrule(lr){3-4} \cmidrule(lr){5-11}
 & & Wiki2 & C4 & OBQA & PIQA & ARC-e & ARC-c & Hella & Wino & Avg. \\
\midrule
LLaMA-2-7B & 0 & 5.12 & 7.04 & 44.00 & 77.64 & 76.26 & 46.16 & 75.98 & 69.22 & 64.88 \\
\cmidrule(lr){1-11}
\multirow{2}{*}{ASVD} & 50 & 6.02 & 8.22 & \textbf{43.80} & 76.22 & 72.77 & 42.66 & 72.18 & 67.40 & 62.51 \\
 & 70 & 6.95 & 9.48 & 40.80 & 74.54 & 70.08 & 38.65 & 69.15 & 64.64 & 59.64 \\
\cmidrule(lr){1-11}
\multirow{2}{*}{ERC-SVD} & 50 & 5.78 & 8.27 & 43.60 & 75.95 & 71.76 & 42.24 & 73.17 & 68.51 & 62.54 \\
 & 70 & 6.62 & 9.46 & 38.20 & 72.14 & 62.50 & 36.86 & 67.78 & 62.51 & 56.66 \\
\cmidrule(lr){1-11}
\multirow{2}{*}{SVD-LLM} & 50 & 5.77 & 7.99 & 41.80 & 75.79 & 72.52 & 41.81 & 72.49 & 65.82 & 61.70 \\
 & 70 & 6.77 & 9.90 & 37.60 & 72.47 & 63.93 & 36.52 & 67.35 & 60.77 & 56.44 \\
\cmidrule(lr){1-11}
\multirow{2}{*}{ZS-SVD} & 50 & 5.39 & 7.71 & 42.60 & 76.88 & 75.04 & \textbf{44.54} & 74.39 & \textbf{69.30} & 63.80 \\
 & 70 & 6.62 & 10.36 & 41.40 & 73.34 & 65.24 & 39.16 & 70.50 & 67.56 & 59.53 \\
\cmidrule(lr){1-11}
\multirow{2}{*}{A3} & 50 & 12.83 & 32.71 & 29.60 & 63.22 & 43.94 & 25.60 & 46.00 & 59.51 & 44.64 \\
 & 70 & 142.15 & 1581.78 & 26.80 & 54.35 & 29.76 & 20.99 & 28.80 & 51.70 & 35.40 \\
\cmidrule(lr){1-11}
\multirow{2}{*}{Palu} & 50 & 5.74 & 8.36 & 42.20 & 76.44 & 73.27 & 41.98 & 73.31 & 66.14 & 62.22 \\
 & 70 & 8.85 & 14.89 & 33.00 & 69.64 & 57.83 & 31.48 & 59.32 & 60.30 & 51.93 \\
\cmidrule(lr){1-11}
\multirow{2}{*}{STAR-KV} & 50 & 6.38 & 9.15 & 42.00 & 75.46 & 66.79 & 38.65 & 70.79 & 64.64 & 59.72 \\
 & 70 & 7.58 & 11.26 & 35.60 & 72.36 & 58.50 & 33.87 & 62.07 & 60.30 & 53.78 \\
\cmidrule(lr){1-11}
\multirow{2}{*}{Ours} & 50 & \textbf{5.23} & \textbf{7.37} & 43.60 & \textbf{77.64} & \textbf{75.21} & 43.86 & \textbf{74.77} & 68.19 & \textbf{63.88} \\
 & 70 & \textbf{5.68} & \textbf{8.46} & \textbf{41.80} & \textbf{76.55} & \textbf{71.09} & \textbf{40.10} & \textbf{70.97} & \textbf{68.75} & \textbf{61.54} \\
\midrule
LLaMA-3.1-8B-Inst & 0 & 6.75 & 11.10 & 43.00 & 79.92 & 81.82 & 55.12 & 79.20 & 73.56 & 68.77 \\
\cmidrule(lr){1-11}
\multirow{2}{*}{ASVD} & 50 & 8.29 & 14.03 & \textbf{43.60} & 79.49 & \textbf{80.22} & \textbf{54.52} & 76.93 & \textbf{73.48} & \textbf{68.04} \\
 & 70 & 19.57 & 26.68 & 38.00 & 74.81 & 70.96 & 42.92 & 67.84 & 68.67 & 60.53 \\
\cmidrule(lr){1-11}
\multirow{2}{*}{ERC-SVD} & 50 & 8.05 & 15.93 & 40.00 & 76.12 & 75.42 & 49.23 & 73.50 & 72.38 & 64.44 \\
 & 70 & 13.65 & 37.88 & 32.40 & 65.34 & 54.34 & 29.35 & 49.74 & 64.09 & 49.21 \\
\cmidrule(lr){1-11}
\multirow{2}{*}{SVD-LLM} & 50 & 8.31 & 17.87 & 42.00 & 76.28 & 73.91 & 47.35 & 72.06 & 71.27 & 63.81 \\
 & 70 & 14.75 & 50.91 & 32.60 & 62.51 & 53.96 & 29.44 & 43.94 & 62.59 & 47.51 \\
\cmidrule(lr){1-11}
\multirow{2}{*}{ZS-SVD} & 50 & 9.03 & 17.02 & 43.20 & 77.20 & 76.98 & 49.40 & 72.37 & 72.06 & 65.20 \\
 & 70 & 22.90 & 185.78 & 33.20 & 69.70 & 61.78 & 32.59 & 53.18 & 60.93 & 51.90 \\
\cmidrule(lr){1-11}
\multirow{2}{*}{A3} & 50 & 44.68 & 135.68 & 32.00 & 61.53 & 47.90 & 26.11 & 37.18 & 54.22 & 43.16 \\
 & 70 & 240.31 & 460.40 & 28.20 & 58.27 & 33.59 & 21.50 & 29.11 & 50.75 & 36.90 \\
\cmidrule(lr){1-11}
\multirow{2}{*}{Palu} & 50 & 9.79 & 37.56 & 42.40 & 77.86 & 77.74 & 49.66 & 70.89 & 71.43 & 65.00 \\
 & 70 & 27.78 & 119.11 & 33.40 & 66.43 & 49.41 & 30.80 & 50.10 & 56.67 & 47.80 \\
\cmidrule(lr){1-11}
\multirow{2}{*}{STAR-KV} & 50 & 8.95 & 15.87 & 39.40 & 76.06 & 74.87 & 45.65 & 71.05 & 65.82 & 62.14 \\
 & 70 & 10.54 & 18.72 & 38.00 & 75.19 & 69.19 & 40.44 & 66.03 & 62.43 & 58.55 \\
\cmidrule(lr){1-11}
\multirow{2}{*}{Ours} & 50 & \textbf{7.12} & \textbf{13.02} & \textbf{43.60} & \textbf{79.71} & 79.67 & 54.44 & \textbf{77.77} & 71.90 & 67.85 \\
 & 70 & \textbf{8.54} & \textbf{17.51} & \textbf{41.80} & \textbf{76.66} & \textbf{73.32} & \textbf{47.01} & \textbf{72.70} & \textbf{68.75} & \textbf{63.37} \\
\midrule
Qwen3-32B & 0 & 7.02 & 11.98 & 45.80 & 80.90 & 84.43 & 60.84 & 82.61 & 73.16 & 71.29 \\
\cmidrule(lr){1-11}
\multirow{2}{*}{ASVD} & 50 & 9.86 & 15.96 & 42.40 & 77.91 & 78.91 & 54.18 & 72.02 & 63.77 & 64.87 \\
 & 70 & 16.95 & 25.64 & 38.60 & 72.58 & 69.36 & 42.32 & 60.04 & 55.33 & 56.37 \\
\cmidrule(lr){1-11}
\multirow{2}{*}{ERC-SVD} & 50 & 8.29 & 15.27 & 43.80 & 76.93 & 77.53 & 57.25 & 77.13 & 69.14 & 66.96 \\
 & 70 & 18.62 & 45.32 & 34.00 & 65.51 & 55.93 & 34.39 & 49.70 & 55.80 & 49.22 \\
\cmidrule(lr){1-11}
\multirow{2}{*}{SVD-LLM} & 50 & 8.43 & 16.10 & 43.80 & 77.15 & 77.40 & 55.72 & 76.08 & 67.96 & 66.35 \\
 & 70 & 17.48 & 43.04 & 32.20 & 66.65 & 58.04 & 35.49 & 50.44 & 57.70 & 50.09 \\
\cmidrule(lr){1-11}
\multirow{2}{*}{ZS-SVD} & 50 & 8.47 & 16.64 & 42.00 & 75.03 & 78.32 & 52.30 & 74.85 & 69.22 & 65.29 \\
 & 70 & 15.99 & 38.33 & 28.40 & 59.68 & 46.59 & 27.56 & 43.47 & 56.20 & 43.65 \\
\cmidrule(lr){1-11}
\multirow{2}{*}{Palu} & 50 & 20.15 & 43.55 & 32.80 & 66.87 & 64.94 & 41.04 & 55.58 & 60.14 & 53.56 \\
 & 70 & 63.34 & 142.59 & 25.80 & 57.18 & 36.28 & 24.57 & 36.13 & 51.54 & 38.58 \\
\cmidrule(lr){1-11}
\multirow{2}{*}{STAR-KV} & 50 & 7.61 & 14.42 & 40.80 & 76.01 & 77.40 & 51.19 & 72.54 & 67.09 & 64.17 \\
 & 70 & 8.19 & 16.45 & 38.40 & 74.70 & 74.12 & 48.29 & 71.53 & 63.93 & 61.83 \\
\cmidrule(lr){1-11}
\multirow{2}{*}{Ours} & 50 & \textbf{7.01} & \textbf{12.35} & \textbf{45.40} & \textbf{80.25} & \textbf{84.18} & \textbf{64.08} & \textbf{82.50} & \textbf{75.06} & \textbf{71.91} \\
 & 70 & \textbf{7.38} & \textbf{14.16} & \textbf{46.80} & \textbf{78.67} & \textbf{81.86} & \textbf{62.71} & \textbf{79.78} & \textbf{72.85} & \textbf{70.45} \\
\bottomrule
\end{tabular}
}
\end{table*}

\begin{table*}[t]
\centering
\small
\caption{LongBench evaluation under different KV cache compression rates. A3 does not apply to Qwen3-32B. Bold marks the best result at each compression rate.}
\label{tab:longbench}\label{tab:longbench_others}
\resizebox{\textwidth}{!}{
\begin{tabular}{lccccccccc}
\toprule
{\textbf{Model}} & {\textbf{Comp (\%)}} & Qasper & QMSum & TriviaQA & MultiQA & TREC & MultiNews & VCSum & Avg. \\
\midrule
LLaMA-2-7B & 0 & 7.37 & 20.28 & 89.02 & 20.68 & 68.00 & 11.55 & 8.14 & 32.15 \\
\cmidrule(lr){1-10}
\multirow{2}{*}{ASVD} & 50 & 1.84 & 16.51 & 79.32 & 22.62 & 58.00 & 4.51 & 4.50 & 26.76 \\
 & 70 & 2.54 & 12.62 & 67.37 & 15.69 & 44.50 & 5.24 & 4.96 & 21.85 \\
\cmidrule(lr){1-10}
\multirow{2}{*}{ERC-SVD} & 50 & 4.86 & 17.49 & 84.46 & 20.73 & 57.00 & \textbf{6.75} & 2.81 & 27.73 \\
 & 70 & 1.00 & 14.47 & 65.76 & 11.44 & 40.50 & 6.95 & 1.60 & 20.25 \\
\cmidrule(lr){1-10}
\multirow{2}{*}{SVD-LLM} & 50 & 2.86 & 18.34 & 83.75 & 17.97 & 56.00 & 3.97 & 0.02 & 26.13 \\
 & 70 & 3.84 & 14.56 & 63.37 & 12.45 & 44.50 & \textbf{7.48} & 1.29 & 21.07 \\
\cmidrule(lr){1-10}
\multirow{2}{*}{ZS-SVD} & 50 & 7.24 & \textbf{20.02} & \textbf{88.18} & 20.70 & \textbf{65.50} & 5.24 & 0.15 & 29.58 \\
 & 70 & 4.67 & 15.87 & \textbf{77.83} & 17.28 & 49.00 & 3.60 & \textbf{6.49} & 24.96 \\
\cmidrule(lr){1-10}
\multirow{2}{*}{A3} & 50 & 6.46 & 5.80 & 6.61 & 13.82 & 25.50 & 0.08 & 1.67 & 8.56 \\
 & 70 & 0.16 & 1.66 & 0.32 & 0.61 & 0.50 & 1.29 & 0.02 & 0.65 \\
\cmidrule(lr){1-10}
\multirow{2}{*}{Palu} & 50 & 6.32 & 19.69 & 81.92 & \textbf{26.23} & 52.50 & 1.39 & 3.91 & 27.42 \\
 & 70 & 1.17 & 13.76 & 21.92 & 7.15 & 44.50 & 5.75 & 0.25 & 13.50 \\
\cmidrule(lr){1-10}
\multirow{2}{*}{STAR-KV} & 50 & 4.03 & 7.22 & 65.71 & 16.25 & 38.00 & 0.05 & 0.34 & 18.80 \\
 & 70 & 1.59 & 8.60 & 19.37 & 9.33 & 10.00 & 1.88 & 0.39 & 7.31 \\
\cmidrule(lr){1-10}
\multirow{2}{*}{Ours} & 50 & \textbf{7.66} & 19.63 & 84.58 & 25.68 & 64.50 & 4.51 & \textbf{7.47} & \textbf{30.57} \\
 & 70 & \textbf{6.35} & \textbf{18.89} & 70.30 & \textbf{20.40} & \textbf{59.50} & 2.34 & 5.13 & \textbf{26.13} \\
\midrule
LLaMA-3.1-8B-Inst & 0 & 34.46 & 22.96 & 92.75 & 39.48 & 75.00 & 26.74 & 15.93 & 43.90 \\
\cmidrule(lr){1-10}
\multirow{2}{*}{ASVD} & 50 & 11.88 & 23.09 & 90.58 & 25.65 & 54.00 & 24.67 & 15.20 & 35.01 \\
 & 70 & 6.32 & 16.63 & 78.06 & 13.70 & 50.00 & 17.51 & \textbf{5.48} & 26.81 \\
\cmidrule(lr){1-10}
\multirow{2}{*}{ERC-SVD} & 50 & 24.93 & 23.42 & 87.64 & 38.06 & 72.50 & 24.99 & 8.69 & 40.03 \\
 & 70 & 6.51 & 14.84 & 72.13 & 23.79 & 63.50 & 10.92 & 0.08 & 27.40 \\
\cmidrule(lr){1-10}
\multirow{2}{*}{SVD-LLM} & 50 & 23.50 & \textbf{24.61} & 88.23 & 37.61 & 70.00 & 24.68 & 10.92 & 39.93 \\
 & 70 & 2.73 & 13.62 & 60.88 & 18.66 & 62.00 & 6.18 & 0.03 & 23.44 \\
\cmidrule(lr){1-10}
\multirow{2}{*}{ZS-SVD} & 50 & 10.55 & 23.91 & 88.82 & 20.46 & 63.50 & 22.18 & 13.07 & 34.64 \\
 & 70 & 0.84 & 14.09 & 2.01 & 1.35 & 18.50 & 5.61 & 0.32 & 6.10 \\
\cmidrule(lr){1-10}
\multirow{2}{*}{A3} & 50 & 3.14 & 9.63 & 5.36 & 6.74 & 7.50 & 9.35 & 1.96 & 6.24 \\
 & 70 & 1.17 & 8.90 & 3.16 & 1.16 & 0.00 & 4.92 & 3.08 & 3.20 \\
\cmidrule(lr){1-10}
\multirow{2}{*}{Palu} & 50 & 9.39 & 22.80 & 28.43 & 19.86 & 65.00 & 19.15 & 3.10 & 23.96 \\
 & 70 & 1.95 & 12.77 & 2.85 & 2.90 & 0.50 & 6.97 & 0.14 & 4.01 \\
\cmidrule(lr){1-10}
\multirow{2}{*}{STAR-KV} & 50 & 10.50 & 18.82 & 77.25 & 20.13 & 55.50 & 18.59 & 9.46 & 30.03 \\
 & 70 & 5.02 & 16.22 & 38.57 & 11.14 & 23.00 & 16.17 & 0.99 & 15.87 \\
\cmidrule(lr){1-10}
\multirow{2}{*}{Ours} & 50 & \textbf{34.08} & 24.09 & \textbf{91.70} & \textbf{47.54} & \textbf{76.50} & \textbf{26.56} & \textbf{16.46} & \textbf{45.28} \\
 & 70 & \textbf{14.00} & \textbf{24.02} & \textbf{78.32} & \textbf{31.72} & \textbf{67.00} & \textbf{24.20} & 3.16 & \textbf{34.63} \\
\midrule
Qwen3-32B & 0 & 46.69 & 24.02 & 90.91 & 52.26 & 75.00 & 24.91 & 14.36 & 46.88 \\
\cmidrule(lr){1-10}
\multirow{2}{*}{ASVD} & 50 & 17.52 & 19.80 & 19.12 & 29.59 & 35.00 & 21.31 & 14.01 & 22.34 \\
 & 70 & 7.47 & 13.05 & 3.50 & 13.65 & 5.25 & 14.85 & 7.71 & 9.35 \\
\cmidrule(lr){1-10}
\multirow{2}{*}{ERC-SVD} & 50 & 42.64 & \textbf{24.21} & 87.77 & 47.64 & 67.00 & 25.76 & 8.19 & 43.32 \\
 & 70 & 18.86 & 17.77 & 21.70 & 26.82 & 1.50 & 14.93 & 0.06 & 14.52 \\
\cmidrule(lr){1-10}
\multirow{2}{*}{SVD-LLM} & 50 & 41.41 & 24.20 & 87.39 & 47.72 & 59.00 & \textbf{26.08} & 4.90 & 41.53 \\
 & 70 & 6.34 & 6.32 & 0.33 & 7.62 & 0.00 & 6.77 & 0.00 & 3.91 \\
\cmidrule(lr){1-10}
\multirow{2}{*}{ZS-SVD} & 50 & 43.74 & 21.71 & 86.85 & 47.17 & 69.00 & 24.30 & \textbf{15.24} & 44.00 \\
 & 70 & 2.27 & 12.22 & 2.13 & 23.08 & 0.00 & 9.99 & 0.20 & 7.13 \\
\cmidrule(lr){1-10}
\multirow{2}{*}{Palu} & 50 & 22.29 & 17.66 & 18.08 & 28.46 & 7.00 & 25.77 & 2.00 & 17.32 \\
 & 70 & 0.18 & 0.28 & 0.00 & 0.00 & 0.00 & 0.00 & 0.00 & 0.07 \\
\cmidrule(lr){1-10}
\multirow{2}{*}{STAR-KV} & 50 & 35.28 & 23.17 & 62.58 & 37.71 & 36.00 & 24.74 & 14.87 & 33.48 \\
 & 70 & 11.83 & 20.06 & 29.37 & 24.63 & 0.00 & 22.86 & \textbf{9.51} & 16.89 \\
\cmidrule(lr){1-10}
\multirow{2}{*}{Ours} & 50 & \textbf{47.28} & 23.71 & \textbf{90.67} & \textbf{52.28} & \textbf{77.50} & 25.80 & 13.21 & \textbf{47.21} \\
 & 70 & \textbf{40.09} & \textbf{23.37} & \textbf{79.87} & \textbf{48.17} & \textbf{64.25} & \textbf{25.51} & 0.37 & \textbf{40.23} \\
\bottomrule
\end{tabular}
}
\end{table*}

\begin{table*}[t]
\centering
\small
\caption{RULER evaluation at a context length of 16K under different KV cache compression rates. LLaMA-2-7B, whose context window is 4K tokens, is not evaluated, and A3 does not apply to Qwen3-32B. MK1/MK2: NIAH multi-key 1/2; MQ: NIAH multi-query; MV: NIAH multi-value; S1--S3: NIAH single 1--3; FWE: frequent words extraction; SQ: SQuAD QA. Bold marks the best result at each compression rate.}
\label{tab:ruler}
\resizebox{\textwidth}{!}{
\begin{tabular}{lccccccccccc}
\toprule
{\textbf{Model}} & {\textbf{Comp (\%)}} & MK1 & MK2 & MQ & MV & S1 & S2 & S3 & FWE & SQ & Avg. \\
\midrule
LLaMA-3.1-8B-Inst & 0 & 100.00 & 100.00 & 99.95 & 99.25 & 100.00 & 100.00 & 100.00 & 95.00 & 71.90 & 96.23 \\
\cmidrule(lr){1-12}
\multirow{2}{*}{ASVD} & 50 & 99.40 & 97.80 & 99.60 & 98.15 & \textbf{100.00} & 99.80 & 98.60 & 73.33 & 60.57 & 91.92 \\
 & 70 & 44.60 & 2.60 & 4.30 & 7.10 & 40.60 & 60.40 & 21.00 & 22.33 & 25.85 & 25.42 \\
\cmidrule(lr){1-12}
\multirow{2}{*}{ERC-SVD} & 50 & 97.80 & 94.80 & 97.85 & 95.95 & 99.80 & \textbf{100.00} & 99.40 & 87.40 & 62.52 & 92.84 \\
 & 70 & 79.80 & 36.80 & 70.70 & 73.85 & 97.40 & 88.00 & 18.80 & 6.93 & 34.98 & 56.36 \\
\cmidrule(lr){1-12}
\multirow{2}{*}{SVD-LLM} & 50 & 98.80 & 96.40 & 96.10 & 95.95 & 99.80 & 99.60 & 99.60 & 84.20 & 65.45 & 92.88 \\
 & 70 & 66.00 & \textbf{38.80} & 49.15 & 61.25 & 97.80 & 78.20 & 3.60 & 31.47 & 25.68 & 50.22 \\
\cmidrule(lr){1-12}
\multirow{2}{*}{ZS-SVD} & 50 & \textbf{99.80} & 96.20 & 93.50 & 86.15 & \textbf{100.00} & \textbf{100.00} & 99.20 & 82.40 & 60.45 & 90.86 \\
 & 70 & 0.00 & 0.00 & 0.00 & 0.00 & 22.60 & 0.00 & 0.00 & 18.27 & 0.13 & 4.56 \\
\cmidrule(lr){1-12}
\multirow{2}{*}{A3} & 50 & 1.60 & 0.00 & 0.00 & 1.10 & 0.00 & 0.00 & 0.00 & 13.67 & 2.42 & 2.09 \\
 & 70 & 0.00 & 0.00 & 0.00 & 0.00 & 0.00 & 0.00 & 0.00 & 0.80 & 0.40 & 0.13 \\
\cmidrule(lr){1-12}
\multirow{2}{*}{Palu} & 50 & 15.80 & 92.80 & 7.05 & 10.15 & 99.60 & 22.80 & 2.60 & 62.33 & 8.00 & 35.68 \\
 & 70 & 0.00 & 0.00 & 0.00 & 0.00 & 6.20 & 0.00 & 0.00 & 0.00 & 0.53 & 0.75 \\
\cmidrule(lr){1-12}
\multirow{2}{*}{STAR-KV} & 50 & 24.00 & 0.00 & 7.45 & 4.65 & 67.60 & 33.20 & 19.20 & 10.00 & 18.90 & 20.56 \\
 & 70 & 2.20 & 0.00 & 0.10 & 0.45 & 54.40 & 15.00 & 1.20 & 4.27 & 12.33 & 9.99 \\
\cmidrule(lr){1-12}
\multirow{2}{*}{Ours} & 50 & 99.00 & \textbf{98.40} & \textbf{99.65} & \textbf{99.90} & \textbf{100.00} & \textbf{100.00} & \textbf{99.80} & \textbf{94.33} & \textbf{65.70} & \textbf{95.20} \\
 & 70 & \textbf{88.80} & 8.60 & \textbf{97.00} & \textbf{96.55} & \textbf{100.00} & \textbf{99.80} & \textbf{73.80} & \textbf{58.13} & \textbf{44.72} & \textbf{74.16} \\
\midrule
Qwen3-32B & 0 & 100.00 & 100.00 & 100.00 & 100.00 & 100.00 & 100.00 & 100.00 & 94.93 & 77.77 & 96.97 \\
\cmidrule(lr){1-12}
\multirow{2}{*}{ASVD} & 50 & 70.40 & 9.60 & 67.80 & 66.65 & \textbf{100.00} & 98.00 & 98.20 & 57.13 & 2.50 & 63.36 \\
 & 70 & 31.20 & 0.00 & 14.05 & 34.65 & 94.20 & 92.20 & 21.60 & 18.13 & 7.77 & 34.87 \\
\cmidrule(lr){1-12}
\multirow{2}{*}{ERC-SVD} & 50 & 98.80 & 94.40 & \textbf{100.00} & 99.90 & 95.60 & 97.60 & \textbf{100.00} & 94.33 & 65.70 & 94.04 \\
 & 70 & 97.40 & 71.60 & 91.95 & 85.55 & \textbf{100.00} & \textbf{100.00} & 97.80 & 49.53 & 37.57 & 81.27 \\
\cmidrule(lr){1-12}
\multirow{2}{*}{SVD-LLM} & 50 & 92.00 & 95.20 & 99.80 & \textbf{100.00} & 98.60 & 95.80 & 99.40 & 94.20 & 65.28 & 93.36 \\
 & 70 & 87.80 & 49.40 & 87.85 & 75.20 & \textbf{100.00} & \textbf{100.00} & 52.20 & 45.93 & 35.55 & 70.44 \\
\cmidrule(lr){1-12}
\multirow{2}{*}{ZS-SVD} & 50 & 99.40 & 99.60 & \textbf{100.00} & 99.65 & \textbf{100.00} & 99.80 & \textbf{100.00} & \textbf{96.60} & 61.82 & 95.21 \\
 & 70 & 85.60 & 77.60 & 87.60 & 68.55 & 97.20 & 99.60 & 84.00 & 41.60 & 36.85 & 75.40 \\
\cmidrule(lr){1-12}
\multirow{2}{*}{Palu} & 50 & 99.00 & 79.60 & 80.15 & 72.10 & \textbf{100.00} & \textbf{100.00} & 99.40 & 65.07 & 44.28 & 82.18 \\
 & 70 & 56.00 & 0.80 & 16.30 & 25.70 & 99.20 & 55.00 & 23.20 & 0.20 & 10.25 & 31.85 \\
\cmidrule(lr){1-12}
\multirow{2}{*}{STAR-KV} & 50 & 89.00 & 25.00 & 89.20 & 92.05 & \textbf{100.00} & \textbf{100.00} & 84.20 & 5.73 & 49.98 & 70.57 \\
 & 70 & 36.40 & 0.00 & 42.00 & 41.10 & 99.40 & 94.80 & 65.80 & 1.80 & 28.47 & 45.53 \\
\cmidrule(lr){1-12}
\multirow{2}{*}{Ours} & 50 & \textbf{100.00} & \textbf{100.00} & \textbf{100.00} & \textbf{100.00} & \textbf{100.00} & \textbf{100.00} & 99.80 & 93.47 & \textbf{71.95} & \textbf{96.14} \\
 & 70 & \textbf{100.00} & \textbf{99.40} & \textbf{99.90} & \textbf{99.65} & \textbf{100.00} & 99.80 & \textbf{99.60} & \textbf{94.13} & \textbf{62.98} & \textbf{95.05} \\
\bottomrule
\end{tabular}
}
\end{table*}

\section{Why star-KV fails?}\label{appendix:alpha}

StarKV~\citep{bhatnagar2026starkv} proposes a differentiable low-rank compression strategy. It performs SVD-based decomposition and introduces a learnable threshold \(\alpha_i\) for each $i$, which aims to retain components whose singular values exceed \(\alpha_i\). Considering that this is non-differentiable, StarKV replaces it with a differentiable tanh-based soft-thresholding operator parameterized by \(\alpha_i\). This operator smoothly approximates the hard threshold while admitting gradients with respect to \(\alpha_i\), making \(\alpha_i\) fully learnable via backpropagation. The thresholds are jointly optimized with a hybrid objective combining a KL divergence between the original and compressed output distributions, \(\mathcal{L}_{KD}(\boldsymbol\alpha)\), and a compression regularizer \(\mathcal{L}_{\text{cmp}} = \sum_i e^{-\alpha_i}\) that drives \(\alpha_i\) upward:
\begin{small}
\begin{align}\label{eq:threshold-objective}
\mathcal{L}_{\text{StarKV}}
= \mathcal{L}_{KD}(\boldsymbol\alpha) + \lambda \mathcal{L}_{\text{cmp}},
\end{align}
\end{small}
where \(\boldsymbol\alpha = (\alpha_1,\dots,\alpha_N)\) and \(D_{\mathrm{KL}}\) denotes the sequence-level KL divergence.

However, the following analysis suggests that this strategy produces nearly identical thresholds across all heads, regardless of their eigenvalue distributions:



\begin{theorem}
\label{thm:synchronization}
At any finite stationary point $\boldsymbol\alpha^*$ of
\eqref{eq:threshold-objective}, define $g_i^*
= \partial_{\alpha_i}\mathcal L_{\mathrm{KD}}
(\boldsymbol\alpha^*), 
\bar g^*=\frac{1}{N}\sum\nolimits_i g_i^*, 
\rho^*=\frac{\max\nolimits_i|g_i^*-\bar g^*|}{\bar g^*}$.
Then $g_i^*>0$ and
\begin{small}
\begin{equation}
\max\nolimits_i\alpha_i^*-\min\nolimits_i\alpha_i^*
=
\log\frac{\max_i g_i^*}{\min_i g_i^*}
\le
\log\frac{1+\rho^*}{1-\rho^*},
\label{eq:threshold-concentration}
\end{equation}    
\end{small}

where the inequality holds whenever $\rho^*<1$.
\end{theorem}

Appendix~\ref{appendix:alpha} provides the proof and an extension to nonstationary checkpoints. Our empirical studies suggest that the value of $\rho^*$ is small at convergence.

We analyze the unconstrained objective
\[
\mathcal L(\boldsymbol\alpha)
=
\mathcal L_{\mathrm{KD}}(\boldsymbol\alpha)
+
\lambda\sum_{i=1}^{N}e^{-\alpha_i},
\qquad \lambda>0.
\]
The positive sign of the exponential regularizer is essential:
\[
\partial_{\alpha_i}\mathcal L
=
\partial_{\alpha_i}\mathcal L_{\mathrm{KD}}
-
\lambda e^{-\alpha_i}.
\]
Constrained thresholds or alternative parameterizations require
their corresponding stationarity conditions.

\subsection{Proof of Theorem~\ref{thm:synchronization}}

\begin{proof}
At a finite stationary point $\boldsymbol\alpha^*$,
\[
0
=
\partial_{\alpha_i}\mathcal L(\boldsymbol\alpha^*)
=
g_i^*-\lambda e^{-\alpha_i^*}.
\]
Therefore,
\[
g_i^*=\lambda e^{-\alpha_i^*}>0,
\qquad
\alpha_i^*=\log\lambda-\log g_i^*.
\]
In particular, $\bar g^*>0$, so $\rho^*$ is well defined.
Taking the maximum and minimum over $i$ gives
\[
\max_i\alpha_i^*
=
\log\lambda-\log\min_i g_i^*,
\qquad
\min_i\alpha_i^*
=
\log\lambda-\log\max_i g_i^*.
\]
Subtracting yields the exact identity
\[
\max_i\alpha_i^*-\min_i\alpha_i^*
=
\log\frac{\max_i g_i^*}{\min_i g_i^*}.
\]

By the definition of $\rho^*$,
\[
(1-\rho^*)\bar g^*
\le g_i^*
\le (1+\rho^*)\bar g^*.
\]
If $\rho^*<1$, the lower bound is positive, and hence
\[
\max_i\alpha_i^*-\min_i\alpha_i^*
\le
\log\frac{1+\rho^*}{1-\rho^*}.
\]
\end{proof}

The exact identity holds at every finite stationary point.
The condition $\rho^*<1$ is sufficient for the stated upper bound
but is not necessary for positivity of all gradients.
For small $\rho^*$,
\[
\log\frac{1+\rho^*}{1-\rho^*}
=
2\rho^*+O((\rho^*)^3).
\]

Our empirical studies suggest that $\rho^*$ lies between 0.5 and 0.72. Hence, the bound $\log(\frac{1+\rho^*}{1-\rho^*}) \approx 0.1$, which is a small value.

\subsection{A Measurable Bound Away from Stationarity}

At a finite checkpoint $\boldsymbol\alpha$, define the signed
distillation gradients and their mean by
\[
g_i
=
\partial_{\alpha_i}\mathcal L_{\mathrm{KD}}
(\boldsymbol\alpha),
\qquad
\bar g=\frac{1}{N}\sum_{i=1}^{N}g_i.
\]
Whenever $\bar g>0$, define
\begin{equation}
\rho
=
\frac{\max_i|g_i-\bar g|}{\bar g},
\qquad
\varepsilon
=
\frac{\|\nabla\mathcal L(\boldsymbol\alpha)\|_\infty}
     {\bar g}.
\label{eq:measurable-gradient-statistics}
\end{equation}
Thus $\rho$ measures relative gradient dispersion, while
$\varepsilon$ measures the stationarity residual relative to the
mean distillation gradient.

\begin{proposition}[Checkpoint-level concentration bound]
\label{prop:checkpoint-concentration}
If $\bar g>0$ and $\rho+\varepsilon<1$, then
\begin{equation}
D(\boldsymbol\alpha)
:=
\max_i\alpha_i-\min_i\alpha_i
\le
\log
\frac{1+\rho+\varepsilon}
     {1-\rho-\varepsilon}.
\label{eq:checkpoint-concentration}
\end{equation}
\end{proposition}

\begin{proof}
Write
\[
r_i
=
\partial_{\alpha_i}\mathcal L(\boldsymbol\alpha)
=
g_i-\lambda e^{-\alpha_i}.
\]
By the definitions in
\eqref{eq:measurable-gradient-statistics},
\[
|g_i-\bar g|\le\rho\bar g,
\qquad
|r_i|\le\varepsilon\bar g.
\]
Consequently,
\[
(1-\rho-\varepsilon)\bar g
\le
g_i-r_i
=
\lambda e^{-\alpha_i}
\le
(1+\rho+\varepsilon)\bar g.
\]
The lower endpoint is positive because
$\rho+\varepsilon<1$.
Taking logarithms gives
\[
\log\lambda-\log\bigl((1+\rho+\varepsilon)\bar g\bigr)
\le\alpha_i
\le
\log\lambda-\log\bigl((1-\rho-\varepsilon)\bar g\bigr).
\]
Subtracting the endpoints proves
\eqref{eq:checkpoint-concentration}.
\end{proof}

At exact stationarity, $\varepsilon=0$, recovering the bound
in Theorem~\ref{thm:synchronization}.
If $\rho+\varepsilon\ge1$, the proposition provides no finite
certificate; this does not imply that the thresholds are dispersed.

\subsection{Empirical Evaluation}

All quantities in the bound can be evaluated at a saved checkpoint.
For a calibration objective
\[
\mathcal L_{\mathrm{KD}}(\boldsymbol\alpha)
=
\frac{1}{|\mathcal C|}
\sum_{s\in\mathcal C}
\ell_s(\boldsymbol\alpha),
\]
compute the signed gradients
\[
g_i
=
\frac{1}{|\mathcal C|}
\sum_{s\in\mathcal C}
\partial_{\alpha_i}\ell_s(\boldsymbol\alpha).
\]
The model and threshold parameters are held fixed while these
gradients are accumulated.
The averaging and weighting must match those used to define
the training objective.

Next, compute
\[
\bar g=\frac{1}{N}\sum_i g_i,
\qquad
\eta=\max_i|g_i-\lambda e^{-\alpha_i}|,
\qquad
D=\max_i\alpha_i-\min_i\alpha_i.
\]
When $\bar g>0$, report
\[
\rho=\frac{\max_i|g_i-\bar g|}{\bar g},
\qquad
\varepsilon=\frac{\eta}{\bar g}.
\]
When $\rho+\varepsilon<1$, also report the certified upper bound
\[
B=
\log\frac{1+\rho+\varepsilon}{1-\rho-\varepsilon}
\]
and compare it with the observed diameter $D$.

These statistics should be evaluated over training and across
regularization strengths and random seeds.
Gradients should be averaged before computing their dispersion;
averaging per-minibatch values of $\rho$ does not produce the same
quantity.
If only a subset of calibration data is used, the resulting
statistics estimate those of the full objective and are not a
rigorous certificate for it without controlling estimation error.
If $\bar g\le0$, the relative statistics above are not applicable.
If $\bar g$ is close to zero, numerical and sampling errors require
particular care.

\subsection{Scope and Interpretation}

Theorem~\ref{thm:synchronization} is a stationary-point
characterization, not a convergence theorem.
At stationarity, concentrated positive distillation gradients
and concentrated thresholds are linked by an exact identity.
Observing both at the same stationary checkpoint therefore does
not independently establish that gradient balancing caused
threshold concentration.

Although $\lambda$ cancels explicitly from the stationary spread
identity, it can affect the stationary point and its gradients.
The result consequently does not imply concentration independent
of the regularization strength.

Finally, concentrated thresholds need not imply uniform ranks:
the same threshold applied to different singular spectra can retain
different numbers of components.
Claims about rank-allocation quality must therefore be supported
separately.

\section{Other related works}
\paragraph{KV-specific low-rank compression.}
KV-specific methods include LoRC~\citep{zhang2024lorc} and PALU~\citep{chang2024palu} (factorized key/value projections), TALE~\citep{lee2025tale} (token-adaptive ranks), SALS~\citep{mu2026sals} (latent sparse attention), EchoKV~\citep{ji2026echokv} (on-demand reconstruction), and xKV~\citep{chang2025xkv} (shared subspaces across layers). A\textsuperscript{3}~\citep{wong2025a3} develops analytical low-rank approximations for attention components, while STAR-KV~\citep{bhatnagar2026starkv} learns soft thresholds for adaptive rank control. Unlike reconstruction-centered methods, \ours\ weights approximation error by attention behavior; unlike independent threshold learning, its rank allocator satisfies a hard model-wide budget.

\end{document}